%% file: neurips-sgd-nonconvexity.tex
\documentclass{article}

\usepackage[preprint]{neurips_2022}

\usepackage[utf8]{inputenc} 
\usepackage[T1]{fontenc}    
\usepackage{hyperref}       
\usepackage{url}            
\usepackage{booktabs}       
\usepackage{amsfonts}       
\usepackage{nicefrac}       
\usepackage{microtype}      
\usepackage{xcolor}         
\usepackage{tablefootnote}

\usepackage{amsmath}
\usepackage{amsthm}              
{

	\newtheorem{proposition}{Proposition}
	\newtheorem{lemma}{Lemma}
	\newtheorem{theorem}{Theorem}
	\newtheorem{corollary}{Corollary}
	
	\theoremstyle{remark}
	\newtheorem{remark}{Remark}
	
	\theoremstyle{definition}
	\newtheorem{definition}{Definition}
	\newtheorem{assumption}{Assumption}
	\newtheorem{property}{Property}
}

\usepackage{hyperref}[6.83]
\hypersetup{colorlinks,
	linkcolor=[rgb]{.61,0,0.3},
	citecolor=[rgb]{.14,.47,.14}}

\usepackage[capitalize,nameinlink]{cleveref}[0.19]
\usepackage{graphicx}
\usepackage{tikz}
\usetikzlibrary{arrows,shapes,arrows.meta}
\usepackage{enumitem}
\usepackage[most]{tcolorbox}
\newtcolorbox{simplebox}[1][]{%
  sharp corners,
  enhanced,
  colback=white,
  attach title to upper,
  #1
}

\crefname{assumption}{Assumption}{Assumptions}
\crefname{theorem}{Theorem}{Theorems}
\crefname{remark}{Remark}{Remarks}
\crefname{lemma}{Lemma}{Lemmas}
\crefname{corollary}{Corollary}{Corollaries}
\crefname{proposition}{Proposition}{Propositions}
\crefname{definition}{Definition}{Definitions}
\crefname{property}{Property}{Properties}

\crefformat{equation}{\textup{#2(#1)#3}}
\crefrangeformat{equation}{\textup{#3(#1)#4--#5(#2)#6}}
\crefmultiformat{equation}{\textup{#2(#1)#3}}{ and \textup{#2(#1)#3}}
{, \textup{#2(#1)#3}}{, and \textup{#2(#1)#3}}
\crefrangemultiformat{equation}{\textup{#3(#1)#4--#5(#2)#6}}%
{ and \textup{#3(#1)#4--#5(#2)#6}}{, \textup{#3(#1)#4--#5(#2)#6}}{, and \textup{#3(#1)#4--#5(#2)#6}}

\Crefformat{equation}{#2Equation~\textup{(#1)}#3}
\Crefrangeformat{equation}{Equations~\textup{#3(#1)#4--#5(#2)#6}}
\Crefmultiformat{equation}{Equations~\textup{#2(#1)#3}}{ and \textup{#2(#1)#3}}
{, \textup{#2(#1)#3}}{, and \textup{#2(#1)#3}}
\Crefrangemultiformat{equation}{Equations~\textup{#3(#1)#4--#5(#2)#6}}%
{ and \textup{#3(#1)#4--#5(#2)#6}}{, \textup{#3(#1)#4--#5(#2)#6}}{, and \textup{#3(#1)#4--#5(#2)#6}}

\crefdefaultlabelformat{#2\textup{#1}#3}

\newcommand{\norm}[1]{\left\Vert #1 \right\Vert}
\newcommand{\inlinenorm}[1]{\Vert #1 \Vert}
\newcommand{\1}[1]{\textbf{1}\left[ #1 \right]}
\newcommand{\Prb}[1]{\mathbb{P}\left[ #1 \right]}
\newcommand{\inlinePrb}[1]{\mathbb{P}[ #1 ]}
\newcommand{\E}[1]{\mathbb{E}\left[ #1 \right]}
\newcommand{\inlineE}[1]{\mathbb{E}[ #1 ]}

\newcommand{\cond}[2]{\mathbb{E}\left[\left. #1 \right\vert #2 \right]}
\newcommand{\condPrb}[2]{\mathbb{P}\left[\left. #1 \right\vert #2 \right]}
\newcommand{\inlinecondPrb}[2]{\mathbb{P}[ #1 \vert #2 ] }

\newcommand{\flb}{F_{l.b.}}
\newcommand{\lesp}{\mathcal{L}_{\epsilon}}

\title{Global Convergence and Stability of Stochastic Gradient Descent}

\author{%
  Vivak Patel \\
  Department of Statistics\\
  University of Wisconsin -- Madison\\
  Madison, WI 53706 \\
  \texttt{vivak.patel@wisc.edu} \\
  \And
  Shushu Zhang \\
  Department of Statistics \\
  University of Michigan -- Ann Arbor \\
  \texttt{shushuz@umich.edu} \\
  \AND
  Bowen Tian \\
  Department of Statistics \\
  The Ohio State University \\
  \texttt{tian.837@buckeyemail.osu.edu} \\
}

\begin{document}

\maketitle

\begin{abstract}
\input{section/abstract}
\end{abstract}

\section{Introduction} \label{section-introduction}
\input{section/introduction}

\section{Stochastic Optimization} \label{section-so}
\input{section/so}

\section{Stochastic Gradient Descent} \label{section-sgd}
\input{section/sgd}

\section{Global Convergence \& Stability} \label{section-convergence-stability}
\input{section/global-convergence-stability}


\section{Conclusion} \label{section-conclusion}
\input{section/conclusion}


\pagebreak
\begin{small}
\bibliographystyle{abbrvnat}
\bibliography{sgd-global-convergence-stability-bib}
\end{small}

\pagebreak 
\section*{Checklist}


\begin{enumerate}

\item For all authors...
\begin{enumerate}
  \item Do the main claims made in the abstract and introduction accurately reflect the paper's contributions and scope?
    \answerYes{Each claim in the abstract corresponds to a statement in the introduction, which then references the corresponding result in the paper.}
  \item Did you describe the limitations of your work?
    \answerYes{See the end of \cref{section-introduction}}
  \item Did you discuss any potential negative societal impacts of your work?
    \answerNA{SGD is already a widely existing and deployed method.}
  \item Have you read the ethics review guidelines and ensured that your paper conforms to them?
    \answerYes{}
\end{enumerate}

\item If you are including theoretical results...
\begin{enumerate}
  \item Did you state the full set of assumptions of all theoretical results?
    \answerYes{Each result references the precise assumptions that are used and these are clearly stated in \cref{section-so}.}
        \item Did you include complete proofs of all theoretical results?
    \answerYes{Outlines are provided in the main document, while complete proofs of each result are provided in the appendix/supplement.}
\end{enumerate}

\item If you ran experiments...
\begin{enumerate}
  \item Did you include the code, data, and instructions needed to reproduce the main experimental results (either in the supplemental material or as a URL)?
    \answerNA{}
  \item Did you specify all the training details (e.g., data splits, hyperparameters, how they were chosen)?
    \answerNA{}
        \item Did you report error bars (e.g., with respect to the random seed after running experiments multiple times)?
    \answerNA{}
        \item Did you include the total amount of compute and the type of resources used (e.g., type of GPUs, internal cluster, or cloud provider)?
    \answerNA{}
\end{enumerate}

\item If you are using existing assets (e.g., code, data, models) or curating/releasing new assets...
\begin{enumerate}
  \item If your work uses existing assets, did you cite the creators?
    \answerNA{}
  \item Did you mention the license of the assets?
    \answerNA{}
  \item Did you include any new assets either in the supplemental material or as a URL?
    \answerNA{}
  \item Did you discuss whether and how consent was obtained from people whose data you're using/curating?
    \answerNA{}
  \item Did you discuss whether the data you are using/curating contains personally identifiable information or offensive content?
    \answerNA{}
\end{enumerate}

\item If you used crowdsourcing or conducted research with human subjects...
\begin{enumerate}
  \item Did you include the full text of instructions given to participants and screenshots, if applicable?
    \answerNA{}
  \item Did you describe any potential participant risks, with links to Institutional Review Board (IRB) approvals, if applicable?
    \answerNA{}
  \item Did you include the estimated hourly wage paid to participants and the total amount spent on participant compensation?
    \answerNA{}
\end{enumerate}

\end{enumerate}


\pagebreak
\appendix
\section{Details for Counter Examples} \label{section-counter-examples-details}
\input{section/examples-details}

\section{Technical Lemmas} \label{section-technical-lemmas}
\input{section/technical_lemmas}

\section{Global Convergence Analysis} \label{section-global-convergence-analysis}
\input{section/global-convergence-analysis}

\section{Stability Analysis} \label{section-stability-analysis}
\input{section/stability-analysis}

\end{document}

%% file: section/abstract.tex
In machine learning, stochastic gradient descent (SGD) is widely deployed to train models using highly non-convex objectives with equally complex noise models. Unfortunately, SGD theory often makes restrictive assumptions that fail to capture the non-convexity of real problems, and almost entirely ignore the complex noise models that exist in practice. In this work, we demonstrate the restrictiveness of these assumptions using three canonical models in machine learning. Then, we develop novel theory to address this shortcoming in two ways. First, we establish that SGD's iterates will either globally converge to a stationary point or diverge under nearly arbitrary nonconvexity and noise models. Under a slightly more restrictive assumption on the joint behavior of the non-convexity and noise model that generalizes current assumptions in the literature, we show that the objective function cannot diverge, even if the iterates diverge. As a consequence of our results, SGD can be applied to a greater range of stochastic optimization problems with confidence about its global convergence behavior and stability.

%% file: section/introduction.tex
Stochastic Gradient Descent (SGD) and its variants are dominant algorithms for solving stochastic optimization problems arising in machine learning, and have expanded their reach to more complex problems from estimating Gaussian Processes \citep{chen2020}, covariance estimation in stochastic filters \citep{kim2021}, and systems identification \citep{hardt2016,zhang2020}. Accordingly, understanding the behavior of SGD and its variants has been crucial to their reliable application in machine learning and beyond. As a result, the theory of these methods has greatly advanced, most notably for SGD as it is the basis for, and simplest of, these methods. Indeed, SGD has been analyzed from many perspectives: global convergence analysis \citep{lei2019,gower2020,khaled2020,mertikopoulos2020,patel2020}, local convergence analysis \citep{mertikopoulos2020}, greedy and global complexity analysis \citep{gower2020,khaled2020}, asymptotic weak convergence \citep{wang2021}, and saddle point analysis \citep{fang2019,mertikopoulos2020,jin2021}.

While all of these perspectives add new dimensions to our understanding of SGD, the global convergence analysis of SGD is the foundation as it dictates whether local analyses, complexity analyses or saddle point analyses are even warranted. As surveyed in \cite{patel2020}, these current global convergence analyses of SGD make a wide variety of assumptions, most commonly: (1) the objective function is bounded from below, (2) the gradient function is globally Lipschitz continuous, (3) the stochastic gradients are unbiased, and (4) the variance of the stochastic gradients are bounded. While the first and third assumption are generally reasonable,\footnote{See \cite{bottou2018} for a simple relaxation of unbiased stochastic gradients.} the second and fourth assumptions \textit{and their more recent generalizations} are not usually applicable to machine learning problems as we now demonstrate through three simple examples. Note, in the first two examples, we make use of a penalty function, which can be removed without impacting the result.

\paragraph{Example 1: Feed Forward Neural Network.} Consider the example $(Y,Z)$ where $Y$ is a binary label and $Z$ is a feature vector. We will attempt to predict $Y$ from $Z$ using a simple multi-layer feed forward network as shown in \cref{figure-ffn}. The next result states that for a simple distribution over the example space and for a simple, archetype network, the gradient function is not globally Lipschitz continuous, nor does it satisfy the (possibly) more general $(L_0,L_1)$-smooth assumption \cite[Definition 1, Assumption 3]{zhang2019}. Moreover, the variance of the stochastic gradients is unbounded.  See \cref{subsection-ffn} for a proof.

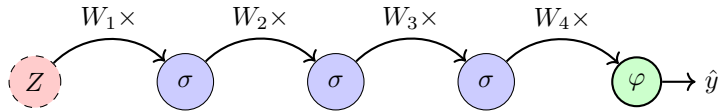
\begin{figure}[hb]
\centering
\begin{tikzpicture}
\node[circle,minimum width=0.5cm,fill=red!20,draw=black,dashed] (feat) at (0,0) {$Z$};
\node[circle,minimum width=0.75cm,fill=blue!20,draw=black] (ly1) at (2,0) {$\sigma$};
\node[circle,minimum width=0.75cm,fill=blue!20,draw=black] (ly2) at (4,0) {$\sigma$};
\node[circle,minimum width=0.75cm,fill=blue!20,draw=black] (ly3) at (6,0) {$\sigma$};
\node[circle,minimum width=0.5cm,fill=green!20,draw=black,thick] (out) at (8,0) {$\varphi$};
\node (y) at (9,0) {$\hat y$};

\draw[->,thick] (feat) to [out=45,in=135] node[above] {$W_1\times$} (ly1);
\draw[->,thick] (ly1) to [out=45,in=135] node[above] {$W_2\times$} (ly2);
\draw[->,thick] (ly2) to [out=45,in=135] node[above] {$W_3\times$} (ly3);
\draw[->,thick] (ly3) to [out=45,in=135] node[above] {$W_4 \times$} (out);
\draw[->,thick] (out) to (y);
\end{tikzpicture}
\caption{A diagram of a simple feed forward network for binary classification.}
\label{figure-ffn}
\end{figure}

\begin{proposition}\label{proposition-ffn}
Consider the feed forward network in \cref{figure-ffn} with $\sigma$ linear and $\varphi$ sigmoid trained with a binary cross-entropy loss with a Ridge penalty. There exists a finite, discrete distribution for $(Y,Z)$ such that the risk function's gradient is not globally Lipschitz continuous, nor does it satisfy the $(L_0,L_1)$-smooth assumption. Moreover, the variance of the stochastic gradients is not bounded.
\end{proposition}

\paragraph{Example 2: Recurrent Neural Network.} Consider the example $(Y,Z_0,Z_1,Z_2,Z_3)$ where $Y$ is a binary label and $\lbrace Z_0,Z_1,Z_2,Z_3 \rbrace$ are sequentially observed. We will attempt to predict $Y$ from $Z$ using a simple recurrent network as shown in \cref{figure-rnn}. The next result states that for a simple distribution over the example space and for a simple, archetype network, the training function violates the aforementioned assumptions. See \cref{subsection-rnn} for a proof.

\begin{figure}[hb]
\centering
\begin{tikzpicture}
\node[circle,minimum width=0.5cm,fill=red!20,draw=black,dashed] (feat) at (0,-2) {$Z_i$};
\node[circle,minimum width=0.5cm,fill=blue!20,draw=black] (act) at (0,0) {$\sigma$};
\node[circle,minimum width=0.5cm,fill=purple!20,draw=black] (state) at (0,2) {$H_i$};
\node[circle,minimum width=0.5cm,fill=green!20,draw=black,thick] (out) at (2,0) {$\varphi$};
\node (y) at (3,0) {$\hat y$};

\draw[->,thick] (feat) to [out=45,in=-45] node[left] {$W_{2} \times$} (act);
\draw[->,thick] (act) to [out=45,in=-45] (state);
\draw[->,thick] (state) to [out=-135,in=135] node[left] {$W_1 \times$} (act);
\draw[->,thick,dashed] (state) to [out=0,in=90] node[right] {$W_3 \times$} (out);
\draw[->,thick] (out) to (y);
\end{tikzpicture}
\caption{A diagram of a recurrent neural network for a binary classification task.}
\label{figure-rnn}
\end{figure}
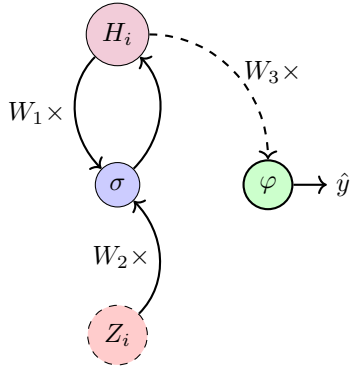

\begin{proposition}\label{proposition-rnn}
Consider the recurrent network in \cref{figure-rnn} with $\sigma$ linear and $\varphi$ sigmoid trained with a binary cross-entropy loss with a Ridge penalty. There exists a finite, discrete distribution for $(Y,Z_0,Z_1,Z_2,Z_3)$ such that the risk function's gradient is not globally Lipschitz continuous, nor does it satisfy the $(L_0,L_1)$-smooth assumption. Moreover, the variance of the stochastic gradients is not bounded.
\end{proposition}

\paragraph{Example 3: Poisson Regression.} Consider fitting a Poisson regression model by the standard maximum likelihood framework using independent copies of the example $(Y,Z)$, where $Y$ is a count response variable and $Z$ is a predictor. The next result states that for a very nice Poisson regression problem, the stochastic gradients violate the bounded variance assumption, its generalization \cite[Assumption 4.3c]{bottou2018}, and, in turn, its generalization, expected smoothness \cite[Assumption 2]{khaled2020}. See \cref{subsection-poisson} for a proof.

\begin{proposition} \label{proposition-pr}
Let $Y$ and $Z$ be independent Poisson random variables with mean one. Consider estimating a Poisson regression model of $Y$ as a function of $Z$. Then, the gradient function is not globally Lipschitz continuous, nor does it satisfy the $(L_0,L_1)$-smooth assumption. Moreover, the variance of the stochastic gradients is not bounded, does not satisfy \cite[Assumption 4.3c]{bottou2018}, nor does it satisfy \cite[Assumption 2]{khaled2020}.
\end{proposition}

As these three examples show, global convergence analyses that make use of the aforementioned assumptions do not apply to these canonical examples of machine learning problems. In fact, to our knowledge and as summarized in \cref{table-survey}, \textit{there are no global convergence analyses of SGD} that apply to these examples. 
\begin{simplebox}
\begin{center}
\textbf{The Problem:} \\

As a result, we do not have guarantees about how SGD behaves on these simple machine learning problems, which calls into question what SGD and its variants are doing on more complicated machine learning models.
\end{center}
\end{simplebox}

\paragraph{Contributions.} To address this problem, 
\begin{enumerate}[leftmargin=*,itemsep=0em,topsep=0em]
\item We relax the global Lipschitz continuous gradient assumption, the bounded variance assumption, and their aforementioned generalizations to assumptions that are applicable to the examples above. Specifically, we relax the  global Lipschitz continuity assumption and the $(L_0,L_1)$-smooth assumption to local $\alpha$-H\"{o}lder continuity of the gradient for $\alpha \in (0,1]$ (see \cref{assumption-local-holder}), which is even a relaxation even for deterministic gradient algorithms (c.f. \cite{nocedal2006},Theorems 3.2, 3.8, 4.5, 4.6). For the $\alpha$ in the local H\"{o}lder assumption, we also relax the bounded variance assumption to only require that the $(1+\alpha)$-moment of the stochastic gradient is bounded by an \textit{arbitrary} upper semi-continuous function (see \cref{assumption-moment}). Our assumption allows stochastic gradients whose noise may not have a variance. Moreover, our assumption generalizes the noise assumption of \cite{bottou2018}, the expected smoothness assumption of \cite{gower2020} and \cite{khaled2020}, and the noise assumption of \cite{asi2019}. We also point out that we do not require the coercivity or asymptotic flatness assumptions that are commonly considered in the analysis of SGD (e.g., \cite{mertikopoulos2020}, Assumptions 2 and 3).
\item Owing to the relaxation in the assumptions, we cannot apply the standard analysis as the local H\"{o}lder constant and the iterate difference are conditionally dependent random variables (see the discussion after \cref{lemma-taylor-holder-bound}). As a result, by generalizing our previous techniques in \cite{patel2020,patel2021stochastic} to the $\alpha$-H\"{o}lder continuous setting, we innovate a new analysis strategy (see \cref{subsection-pseudo-global}) to prove that, with probability one, either SGD's iterates will converge to a stationary point \textit{or} they diverge (see \cref{theorem-global-convergence}). Importantly, our new analysis strategy can be broadly applied even to deterministic algorithms to relax the assumptions found in the literature.
\item The divergence component of our \cref{theorem-global-convergence} is somewhat disconcerting as we cannot say exactly what happens if the iterates diverge. To ameliorate this concern, we add an additional assumption (see \cref{assumption-stability}) and introduce another analysis strategy (see \cref{subsection-local}) to prove, even if the iterates are diverging, the objective function converges to a finite random variable with probability one (see \cref{theorem-stability}).
\end{enumerate}
Our results \textit{do not} supply a rate of convergence as this is impossible for the broad class of nonconvex functions and the generality of the noise models studied in this work \citep{wolpert1997}; in other words, we can always construct a nonconvex objective function, a noise process, and choose an initialization such that any rate of convergence statement will be violated. Indeed, we find it remarkable that it is even possible to provide a global convergence statement for such a broad class of nonconvex functions and general noise models.

\begin{table}[t]
\centering
\caption{A survey of influential, recent global analyses of SGD and their dependence on the two assumptions that are either individually or both violated by the simple examples discussed in \cref{section-introduction}. \label{table-survey}}
\begin{tabular}{@{}p{1.2in}p{3.5in}} \toprule
\textbf{Assumption} & \textbf{Works Depending on the Assumption} \\ \midrule
Global Lipschitz or H\"{o}lder Continuity of Gradient & \cite{reddi2016,ma2017,zhou2018,bassily2018,lei2019,li2019,gower2020,khaled2020,mertikopoulos2020,patel2020,jin2021,wang2021}. \\
Bounded Variance of Stochastic Gradients & \cite{reddi2016stochastic,ma2017,majewski2018,hu2019,bi2019stochastic,zou2019,mertikopoulos2020}.\tablefootnote{In \cite{majewski2018}, this is implied by their third assumption.} \\ \bottomrule
\end{tabular}
\end{table}

\paragraph{Limitations.} We make note of two important limitations in our work. First, we do not consider the important case of nonsmoothness in this work as we require that the gradients of the stochastic optimization function are continuous. However, we note that if the results of \cite{bianchi2022} are broadly applicable, then SGD never observes a point of nonsmoothness and our results would then be applicable. Second, we do not have a simple interpretation of \cref{assumption-stability}---though it seems to have a close relative in another analysis (see \cite{wang2020})---, nor have we been able to construct a relevant counterexample that can illuminate the limitations of this assumption.

%% file: section/so.tex
We consider solving the optimization problem
\begin{equation} \label{eqn-stochastic-optimization-problem}
\min_{\theta \in \mathbb{R}^p} \lbrace F(\theta) := \E{ f(\theta, X) } \rbrace,
\end{equation}
where $F$ maps $\mathbb{R}^p$ into $\mathbb{R}$; $f$ maps $\mathbb{R}^p$ and the co-domain of the random variable $X$ into $\mathbb{R}$; and $\mathbb{E}$ is the expectation operator. As we require gradients, we take $F$ and $f$ to differentiable in $\theta$, and denote its derivatives with respect to $\theta$ by $\dot F(\theta)$ and $\dot f (\theta, X)$. With this notation, we make the following general assumptions about the deterministic portion of the objective function.

\begin{assumption} \label{assumption-flb}
There exists $\flb \in \mathbb{R}$ such that $\forall \theta \in \mathbb{R}^p$, $\flb \leq F(\theta)$.
\end{assumption}

\begin{assumption} \label{assumption-local-holder}
There exists $\alpha \in (0,1]$ such that $\dot{F}(\theta)$ is locally $\alpha$-H\"{o}lder continuous.
\end{assumption}

\begin{remark}
For $\dot F$ to be locally $\alpha$-H\"{o}lder continuous for some $\alpha \in (0,1]$, for every compact set $K \subset \mathbb{R}^p$ there exists a constant $L \geq 0$ such that for every $\theta, \psi \in K$,
\begin{equation}
\norm{ \dot F(\theta) - \dot F (\psi) }_2 \leq L \norm{ \theta - \psi}_2^\alpha.
\end{equation}
\end{remark}

\begin{remark}
As an example, an empirical risk minimization problem for a deep neural network with twice continuously differentiable activation functions with a twice continuously differentiable loss function will readily satisfy the above conditions.
\end{remark}

\cref{assumption-flb,assumption-local-holder} would even be considered mild in the context of non-convex deterministic optimization, in which it is also common to assume that the objective function has well-behaved level sets \citep[e.g.,][Theorems 3.2, 3.8, 4.5, 4.6]{nocedal2006}. Importantly, \cref{assumption-local-holder} relaxes the common restrictive assumption of globally H\"{o}lder continuous gradient functions that is common in other analyses (see \cref{table-survey}).

Our final step is to make some assumptions about the stochastic portion of the objective function. The first assumption requires the stochastic gradients to be unbiased, which can readily be relaxed \citep{bottou2018}. The second assumption allows for a generic noise model for an $\alpha$-H\"{o}lder continuous gradient function, and even allows for the second moment to not exist when $\alpha < 1$ \citep[c.f.][ which requires a decomposition of the noise term that we do not require]{wang2021}.

\begin{assumption} \label{assumption-unbiased}
For all $\theta \in \mathbb{R}^p$, $\dot F(\theta) = \inlineE{ \dot f(\theta, X ) }$.
\end{assumption}

\begin{assumption} \label{assumption-moment}
Let $\alpha \in (0,1]$ be as in \cref{assumption-local-holder}. There exists an upper semi-continuous function $G(\theta)$ such that $\inlineE{ \inlinenorm{ \dot f(\theta,X) }_2^{1+\alpha}} \leq G(\theta)$.

\begin{remark}
For $G(\theta)$ to be upper semi-continuous, then for all $g > 0$, $\lbrace \theta \in \mathbb{R}^p : G(\theta) < g \rbrace$ are open in $\mathbb{R}^p$.
\end{remark}

\end{assumption}
\begin{simplebox}
\begin{center}
We will show that \cref{assumption-flb,assumption-local-holder,assumption-unbiased,assumption-moment} are sufficient for a global convergence result (see \cref{theorem-global-convergence}). 
\end{center}
\end{simplebox}

\begin{remark}
As shown in \S\ref{section-counter-examples-details}, our examples from \S\ref{section-introduction} satisfy \cref{assumption-flb,assumption-local-holder,assumption-unbiased,assumption-moment}.
\end{remark}

\begin{remark}
It is entirely possible that $\inlineE{ \inlinenorm{ \dot f(\theta,X)}_2^{1+\alpha}}$ is (at least) upper semi-continuous, and to set $G(\theta)$ equal to this function. In the case that $\inlineE{ \inlinenorm{ \dot f(\theta,X)}_2^{1+\alpha}}$ is not upper semi-continuous, it is possible to specify $G(\theta)$ as the upper envelope of $\inlineE{ \inlinenorm{ \dot f(\theta,X)}_2^{1+\alpha}}$ (i.e., its limit supremum function). However, it is unlikely that $\inlineE{ \inlinenorm{ \dot f(\theta,X) }_2^{1+\alpha}}$ nor its upper envelope are easy to specify explicitly, and it is more likely to be able to find an upper bound.
\end{remark}

\begin{remark}
We use \cref{assumption-moment} to conclude that on any compact set, the $1+\alpha$ moment of the stochastic gradient is bounded. Of course, we can assume this directly (i.e., on any compact set, the $1+\alpha$ moment is bounded), which, at first glance, appears to be a relaxation. However, if we assume that on any compact set, the $1+\alpha$ moment is bounded, we can use this to construct a $G(\theta)$ that is upper semi-continuous. Thus, the two assumptions are equivalent.
\end{remark}

\begin{remark}
A simple example that shows the utility of \cref{assumption-moment} is to optimize $\inlineE{ \theta^X }$ where $X$ is an exponential random variable with parameter $1$ and $\theta \in [1,u]$ where $u < \exp(1)$. First, it is easy to confirm that the objective function is differentiable and its derivative is globally Lipschitz continuous. Moreover, given that we are on a bounded interval, we conclude that the derivative is globally $\alpha$-H\"{o}lder continuous for any $\alpha \in (0,1]$; therefore, we are free to choose the $\alpha$ as we see fit. Now, when $u < \exp(1/2)$, we have that second moment of the stochastic gradient function exists. However, when $\exp(1/2) < u < \exp(1)$, only smaller moments of the stochastic gradient will exist. Specifically, only for $u < \exp(1/(1+\epsilon))$ with $\epsilon \in (0,1)$ will the $1+\epsilon$ moment of the stochastic gradient will exist. Thus, depending on the size of our interval, we may not have the existence of the second moment, and, consequently, we may not have the existence of the variance.
\end{remark}

In order to show that the objective function cannot diverge (i.e., to prove stability), we will need an additional assumption. This assumption will relate the gradient function, noise model and variation on the local H\"{o}lder constant. To begin, we define the variation on the local H\"{o}lder constant. Let $\alpha \in (0,1]$ be as in \cref{assumption-local-holder} and $\epsilon > 0$ be arbitrary, and define
\begin{equation} \label{eqn-lesp}
\lesp(\theta) = \begin{cases}
\sup_{\varphi} \left\lbrace \frac{\norm{ \dot F(\varphi) - \dot F(\theta)}_2}{\norm{ \varphi - \theta }_2^\alpha} : \norm{ \varphi - \theta}_2 \leq (G(\theta) \vee \epsilon )^{\frac{1}{1+\alpha}} \right\rbrace & \text{if this quantity is nonzero} \\
\epsilon & \text{otherwise},
\end{cases}
\end{equation}
where $\vee$ indicates the maximum between two quantities. Note, the choice of $\epsilon$ is irrelevant, and they can be distinct for the two cases in the definition of $\lesp$, but we fix them to be the same for simplicity. Importantly, the quantity, $\lesp$, is defined at every parameter $\theta$ under \cref{assumption-local-holder}.

With this quantity, we can state a nonintuitive, technical assumption that is needed to prove stability.
\begin{assumption} \label{assumption-stability}
There exists $C_1, C_2, C_3 \geq 0$ such that, $\forall \theta \in \mathbb{R}^p$,
\begin{equation}
\lesp(\theta) G(\theta) + \alpha \left( \frac{\norm{\dot F(\theta) }_2^{1+\alpha}}{\lesp(\theta)} \right)^{1/\alpha} \leq C_1 + C_2 (F(\theta) - \flb) + C_3 \norm{ \dot F(\theta) }_2^2.
\end{equation}
\end{assumption}

\cref{assumption-stability} generalizes Assumption 4.3(c) of \cite{bottou2018}, which is satisfied for a large swath of statistical models. Moreover, \cref{assumption-stability} generalizes the notion of expected smoothness \citep[see][for a history of the assumption]{khaled2020}, which expanded the optimization problems covered by the theory of \cite{bottou2018}. Note, \cref{assumption-stability} is about the asymptotic properties of the stochastic optimization problem as the left hand side of the inequality in \cref{assumption-stability} can be bounded inside of any compact set. Thus, \cref{assumption-stability} covers a variety of asymptotic behaviors, such as $\exp(\inlinenorm{\theta}_2^2)$, $\exp(\inlinenorm{\theta}_2)$, $\inlinenorm{\theta}_2^r$ for $r \in \mathbb{R}$, $\log(\inlinenorm{\theta}_2 + 1)$, and
$\log(\log(\inlinenorm{\theta}_2+1) + 1)$.
Therefore, \cref{assumption-stability} holds for functions with a variety of different asymptotic behaviors.

\begin{simplebox}
\begin{center}
We will show that \cref{assumption-flb,assumption-local-holder,assumption-unbiased,assumption-moment,assumption-stability} are sufficient for a stability result (see \cref{theorem-stability}).
\end{center}
\end{simplebox}

Now that we have specified the nature of the stochastic optimization problem, we turn our attention to the algorithm used to solve the problem, namely, stochastic gradient descent (SGD).

%% file: section/sgd.tex
SGD starts with an arbitrary initial value, $\theta_0 \in \mathbb{R}^p$, and generates a sequence of iterates $\lbrace \theta_k : k \in \mathbb{N} \rbrace$ according to the rule 
\begin{equation} \label{eqn-sgd-update}
\theta_{k+1} = \theta_k - M_k \dot f(\theta_k, X_{k+1}),
\end{equation}
where $\lbrace M_k : k+1 \in \mathbb{N} \rbrace \subset \mathbb{R}^{ p \times p }$; and $\lbrace X_k : k \in \mathbb{N} \rbrace$ are independent and identically distributed copies of $X$. Importantly, $\lbrace M_k \rbrace$ cannot be arbitrary, and the following properties specify a generalization of the \cite{robbins1951} conditions for matrix-valued learning rates \citep[c.f.][]{patel2020}. 

The first condition requires a positive learning rate, and imposes symmetry to ensure the existence of real eigenvalues.
\begin{property} \label{prop-psd}
$\lbrace M_k : k+1 \in \mathbb{N} \rbrace$ are symmetric, positive definite matrices.
\end{property}
The next two properties are a natural generalization of the Robbins-Monro conditions. Let $\alpha \in (0,1]$ be as in \cref{assumption-local-holder}.
\begin{property} \label{prop-max-eig}
Let $\lambda_{\max}(\cdot)$ denote the largest eigenvalue of a symmetric, positive definite matrix. Then, $\sum_{k=0}^\infty \lambda_{\max}(M_k)^{1+\alpha} =: S < \infty$.
\end{property}
\begin{property} \label{prop-min-eig}
Let $\lambda_{\min}(\cdot)$ denote the smallest eigenvalue of a symmetric, positive definite matrix. Then, $\sum_{k=0}^\infty \lambda_{\min}(M_k) = \infty$.
\end{property}

\begin{simplebox}
\begin{center}
We will show that \cref{prop-psd,prop-max-eig,prop-min-eig} are sufficient for a global convergence result (see \cref{theorem-global-convergence}). 
\end{center}
\end{simplebox}

The final property ensures the stability of the condition number of $\lbrace M_k \rbrace$. Note, this property is readily satisfied for scalar learning rates satisfying the Robbins-Monro conditions.
\begin{property} \label{prop-cond-num}
Let $\kappa(\cdot)$ denote the ratio of the largest and smallest eigenvalues of a symmetric, positive definite matrix. Then, $\lim_{k \to \infty} \lambda_{\max}(M_k)^\alpha \kappa(M_k) = 0$.
\end{property}

\begin{simplebox}
\begin{center}
We will show that \cref{prop-psd,prop-max-eig,prop-min-eig,prop-cond-num} are sufficient for stability (see \cref{theorem-stability}).
\end{center}
\end{simplebox}

%% file: section/global-convergence-stability.tex
With the stochastic optimization problem and with stochastic gradient descent (SGD) specified, we now turn our attention to what happens when SGD is applied to a stochastic optimization problem. The key step in the analysis of SGD on any objective function is to establish a bound between the optimality gap at $\theta_{k+1}$ with that of $\theta_k$. This step is achieved by using the local H\"{o}lder continuity of the gradient function and the fundamental theorem of calculus. Using \cref{assumption-local-holder}, we first specify the local H\"{o}lder constant.
\begin{definition} \label{def-local-holder-constant}
For any $\theta, \varphi \in \mathbb{R}^p$, define
\begin{equation}
L(\theta, \varphi) = \sup_{\psi} \left\lbrace \frac{\norm{ \dot F (\psi) - \dot F (\theta) }_2}{ \norm{ \psi - \theta}_2^\alpha} : \psi \in \overline{B(\theta, \norm{\varphi - \theta}_2)} \right\rbrace,
\end{equation}
where $B(\theta, r)$ is an open ball around $\theta$ of radius $r > 0$, and $\overline{B(\theta,r)}$ is its closure. Moreover, for any $R \geq 0$, let $L_R$ be the supremum of $L(\theta, \varphi)$ for any distinct $\theta,\varphi \in \overline{B(0,R)}$.
\end{definition}
\begin{remark}
Note, when the gradient is locally H\"{o}lder continuous, $L_R$ is finite for any $R \geq 0$.
\end{remark}
With this definition, we can now relate the optimality gap of $\theta_{k+1}$ with that of $\theta_k$ by using the following result and proved in \cref{section-technical-lemmas}.
\begin{lemma} \label{lemma-taylor-holder-bound}
Suppose \cref{assumption-flb,assumption-local-holder} hold. Then, for any $\theta, \varphi \in \mathbb{R}^p$,
\begin{equation}
F(\varphi) - \flb \leq F(\theta) - \flb + \dot F(\theta)'(\varphi - \theta) + \frac{L(\theta, \varphi)}{1+\alpha} \norm{ \varphi - \theta}_2^{1+\alpha}.
\end{equation}
\end{lemma}

Now, if we simply set $\varphi = \theta_{k+1}$ and $\theta = \theta_{k}$ in \cref{lemma-taylor-holder-bound} and try to take expectations to manage the randomness of the stochastic gradient, we will run into the problem that $L(\theta_{k},\theta_{k+1})$ and $\norm{\theta_{k+1} - \theta_k}_2$ are potentially dependent,\footnote{While it is possible that these two terms are independent, we would require a lot more information to determine this and it would likely be on an iterate-by-iterate basis for the general class of problems considered in this work. Thus, in this general setting, we cannot assume independence and need to default to treating these terms as dependent.} and we cannot compute its expectation. In previous work, this technical challenge is waived away by using a global H\"{o}lder constant to upper bound $L(\theta_k, \theta_{k+1})$, which is unrealistic even for simple problems (see \cref{section-introduction}).

To address this technical challenge, we innovate two new strategies for handling the dependence between $L(\theta_k, \theta_{k+1})$ and $\norm{\theta_{k+1} - \theta_k}_2$.  In both strategies, we follow the same general approach:
\begin{enumerate}[noitemsep,parsep=0em,topsep=0em]
\item We begin by restricting our analysis to specific events, which will allow us to decouple $L(\theta_k, \theta_{k+1})$ and $\inlinenorm{\theta_{k+1} - \theta_k}_2$.
\item With these two quantities decoupled, we will develop a recurrence relationship between the optimality gap at $\theta_{k+1}$ and that of $\theta_k$.
\item We apply this recurrence relationship with refinements of standard arguments or new ones to derive the desired property about the objective function.
\item Finally, we state the generality of the specific events on which we have studied SGD's iterates.
\end{enumerate}
Thus, it follows, we will define two distinct series of events for the two strategies. The first strategy, which we refer to as the pseudo-global strategy, will provide the global convergence analysis. The second strategy, which we refer to as the local strategy, will provide the stability result.

\subsection{Pseudo-Global Strategy and Global Convergence Analysis} \label{subsection-pseudo-global}

For the first strategy, which supplies the global convergence result, we study SGD on the events
\begin{equation} \label{eqn-iterates-in-ball}
\mathcal{B}_k(R) := \bigcap_{j=0}^k \left\lbrace \norm{\theta_j}_2 \leq R \right\rbrace, ~k +1 \in \mathbb{N},
\end{equation}
for every $R \geq 0$. We now try to control the optimality gap at iteration $k+1$ with that of iteration $k$, which will result in two cases. 
\begin{enumerate}[noitemsep,parsep=0em,topsep=0em]
\item (Case 1) $\mathcal{B}_{k+1}(R)$ holds. We can bound $L(\theta_k,\theta_{k+1})$ by $L_R$, and $G(\theta)$ is also bounded in the ball of radius $R$ about the origin (which follows from $G$ being upper semi-continuous in \cref{assumption-moment}). As a result, we could then proceed with the analysis in a manner that is similar to having a global H\"{o}lder constant.
\item (Case 2) $\inlinenorm{\theta_{k+1}}_2 > R$ and $\mathcal{B}_{k}(R)$ holds. In this case, controlling $L(\theta_k, \theta_{k+1})$ is very challenging and, to our knowledge, was not solved before our work.
\end{enumerate} 
Our approach for controlling the optimality gap in both cases is supplied in the next lemma, whose proof is in \cref{section-global-convergence-analysis}.

\begin{lemma}\label{lemma-recursion-constrained}
Let $\lbrace M_k \rbrace$ satisfy \cref{prop-psd}.
Suppose \cref{assumption-flb,assumption-local-holder,assumption-unbiased,assumption-moment} hold. 
Let $\lbrace \theta_k \rbrace$ satisfy \cref{eqn-sgd-update}.
Then, $\forall R \geq 0$,
\begin{equation}
\begin{aligned}
&\cond{ [F(\theta_{k+1}) - \flb] \1{ \mathcal{B}_{k+1}(R) } }{\mathcal{F}_k} \leq [F(\theta_k) - \flb] \1{ \mathcal{B}_{k}(R) } \\
&\quad - \lambda_{\min}(M_k) \norm{ \dot{F}(\theta_k)}_2^2 \1{\mathcal{B}_k(R)} + \frac{L_{R+1} + \partial F_R}{1+\alpha} \lambda_{\max}(M_k)^{1+\alpha} G_R,
\end{aligned}
\end{equation}
where $G_R = \sup_{\theta \in \overline{B(0,R)}} G(\theta) < \infty$ with $G(\theta)$; and $\partial F_R  = \sup_{ \theta \in \overline{B(0,R)} } \inlinenorm{ \dot{F}(\theta)}_2 (1+\alpha) < \infty$.
\end{lemma}

With this recursion and standard martingale results \citep[Exercise II.4]{robbins1971,neveu1975}, the limit of $[F(\theta_k) - \flb]\1{ \mathcal{B}_k(R)}$ exists with probability one and is finite for every $R \geq 0$. As a result, the limit of $F(\theta_k) - \flb$ exists and is finite on the event $\lbrace \sup_k \inlinenorm{\theta_k}_2 < \infty \rbrace$ (see \cref{corollary-convergence-objective-constrained}).

We can also use \cref{lemma-recursion-constrained} to make a statement about the gradient. Specifically, we can show that the limit infimum of $\inlineE{ \inlinenorm{ \dot F(\theta_k) }_2^2 \1{\mathcal{B}_k(R)}}$ must be zero, which is now a standard argument that mimics Zoutendijk's theorem \citep[Theorem 3.2]{nocedal2006}. 
By Markov's inequality, this result implies that $\inlinenorm{\dot F(\theta_k)}_2 \1{\mathcal{B}_k(R)}$ gets arbitrarily close to $0$ infinitely often (see \cref{lemma-liminf-grad-constrained}). 
To show convergence to zero, however, is not standard. Several strategies have been developed, namely those of \cite{li2019,lei2019,mertikopoulos2020,patel2020,patel2021stochastic}. Unfortunately, the approaches of \cite{li2019,lei2019} rely intimately on the existence of a global H\"{o}lder constant, while that of \cite{mertikopoulos2020} requires even more restrictive assumptions. Fortunately, the approach of \cite{patel2020,patel2021stochastic} can be improved and generalized to the current context (see \cref{lemma-limsup-grad-constrained}). 
Thus, we show that $\lim_{k \to \infty} \inlinenorm{ \dot F(\theta_k)}_2 = 0$ on $\lbrace \sup_k \inlinenorm{\theta_k}_2 < \infty \rbrace$ (see \cref{corollary-convergence-grad-constrained}).

Our final step is to clarify the role of $\lbrace \sup_k \inlinenorm{\theta_k}_2 < \infty \rbrace$ in the asymptotics of SGD's iterates. At first glance, this event seems to imply that the iterates converge to a point. However, owing to the general nature of the noise, it is also possible, say, that the iterates approach a limit cycle or oscillate between points with the same norm. Even beyond this event, the generality of the noise model may allow for substantial excursions between $\flb$ and infinity (c.f., a simple random walk, which has a limit supremum of infinity and a limit infinimum of negative infinity). Thankfully, we can prove that either the iterates converge to a point or they must diverge---a result that we refer to as the Capture Theorem (see \cref{section-global-convergence-analysis}).

\begin{theorem}[Capture Theorem] \label{theorem-capture}
Let $\lbrace \theta_k \rbrace$ be defined as in \cref{eqn-sgd-update}, and let $\lbrace M_k \rbrace$ satisfy \cref{prop-psd,prop-max-eig}. If \cref{assumption-moment} holds, then either $\lbrace \lim_{k \to \infty} \theta_k ~\mathrm{exists} \rbrace$ or $\lbrace \liminf_{k\to\infty} \inlinenorm{ \theta_k}_2 = \infty \rbrace$ must occur.
\end{theorem}

By putting together the above arguments and results, we can conclude that either SGD's iterates diverge or SGD's iterates converge to a stationary point.

\begin{theorem}[Global Convergence] \label{theorem-global-convergence}
Let $\theta_0$ be arbitrary, and let $\lbrace \theta_k : k \in \mathbb{N} \rbrace$ be defined according to \cref{eqn-sgd-update} with $\lbrace M_k : k+1 \rbrace$ satisfying \cref{prop-psd,prop-max-eig,prop-min-eig}. Suppose \cref{assumption-flb,assumption-local-holder,assumption-unbiased,assumption-moment} hold. Let $\mathcal{A}_1 = \lbrace \liminf_{k\to\infty} \inlinenorm{ \theta_k}_2 = \infty \rbrace$ and $\mathcal{A}_2 = \lbrace \lim_{k \to \infty} \theta_k ~\mathrm{exists} \rbrace$. Then, the following statements hold.
\begin{enumerate}[noitemsep,parsep=0em,topsep=0em]
\item $\inlinePrb{ \mathcal{A}_1 } + \inlinePrb{ \mathcal{A}_2 } = 1$.
\item On $\mathcal{A}_2$, there exists a finite random variable, $F_{\lim}$, such that $\lim_{k \to \infty} F(\theta_k) = F_{\lim}$ and $\lim_{k \to \infty} \dot F(\theta_k) = 0$ with probability one.
\end{enumerate}
\end{theorem}
\begin{proof}
By \cref{theorem-capture}, we have that $\inlinePrb{\mathcal{A}_1} + \inlinePrb{\mathcal{A}_2} = 1$. Then, on $\mathcal{A}_2$, \cref{corollary-convergence-objective-constrained,corollary-convergence-grad-constrained} imply that $F(\theta_k) \to F_{\lim}$, which is finite, and $\dot F(\theta_k) \to 0$.
\end{proof}

We pause to stress to a fact about \cref{theorem-global-convergence}: it is nonobvious. To be specific, under such general nonconvexity and noise, we should anticipate any number of asymptotic behaviors for the iterates: convergence to a stationary point, convergence to a nonstationary point, being trapped in a cycle, convergence to a limit cycle, and divergence to infinity. However and very surprisingly, we are able to show that only two possible outcomes can occur: convergence to a stationary point or divergence. Indeed, in previous results \citep[e.g.,][]{mertikopoulos2020,patel2020}, only less specific determinations could be made under much more limited settings.

\subsection{Local Strategy and Stability Analysis} \label{subsection-local}

While \cref{theorem-global-convergence} provides a complete global convergence result, it allows for the possibility of diverging iterates. The possibility of divergent iterates raises the spectre of whether the objective function can also diverge along this sequence. That is, there is a possibility that SGD may be unstable, which would be highly unexpected and undesirable, especially when the objective function is coercive (e.g., has an $\ell^1$ penalty on the parameters). 
To formalize this concept, we define a relevant notion of stability.
\begin{definition}\label{definition-stability}
Stochastic Gradient Descent is stable if
\begin{equation}
\Prb{ \limsup_{k \to \infty} F(\theta_k) = \infty } = 0,
\end{equation}
where $\lbrace \theta_k \rbrace$ satisfy \cref{eqn-sgd-update}.
\end{definition}

We now state the stopping times that we will use to decouple the relationship between $L(\theta_k, \theta_{k+1})$ and $\inlinenorm{\theta_{k+1} - \theta_k}_2$. For every $j + 1 \in \mathbb{N}$, define
\begin{equation} \label{eqn-tau-j}
\tau_j = \min \left\lbrace k: \begin{aligned}
 &F(\theta_{k+1}) - \flb > F(\theta_k) - \flb + \dot F(\theta_k)' (\theta_{k+1} - \theta_k) \\
 &\quad\quad+ \frac{\lesp(\theta_{k})}{1+\alpha} \norm{ \theta_{k+1} - \theta_k}_2^{1+\alpha},~\mathrm{and}~ k \geq j 
\end{aligned}
\right\rbrace.
\end{equation}

Now, we will use \cref{eqn-tau-j} to establish the stability of the objective function. Just as we did with $\mathcal{B}_k(R)$, we will derive a recursion on the optimality gap over the events $\lbrace \lbrace \tau_j > k \rbrace : k +1 \in \mathbb{N} \rbrace$. Of course, just as before, the main challenge in deriving a recursive formula is to address $\lbrace \tau_j = k \rbrace$. Our solution is supplied in the following lemma, whose proof is in \cref{section-stability-analysis}.

\begin{lemma} \label{lemma-optimality-recursion}
Let $\lbrace M_k \rbrace$ satisfy \cref{prop-psd}.
Suppose \cref{assumption-flb,assumption-local-holder,assumption-unbiased,assumption-moment} hold.
Let  $\lbrace \theta_k \rbrace$ satisfy \cref{eqn-sgd-update}.
Then, for any $j+1 \in \mathbb{N}$ and $k > j$, 
\begin{equation}
\begin{aligned}
&\cond{ \left( F(\theta_{k+1} ) - \flb \right) \1{ \tau_j > k } }{ \mathcal{F}_k} 
\leq \left( F(\theta_{k}) - \flb - \dot F(\theta_k)' M_k \dot F(\theta_k) \right) \1{ \tau_j > k-1} \\
&\quad + \frac{\lambda_{\max}(M_k)^{1+\alpha}}{1+\alpha}\left[\mathcal{L}_{\epsilon}(\theta_k) G(\theta_k) + \alpha \left[\frac{ \norm{ \dot F(\theta_k) }_2^{1+\alpha} }{ \mathcal{L}_{\epsilon}(\theta_k) }  \right]^{1/\alpha}  \right] \1{\tau_j > k - 1}.
\end{aligned}
\end{equation}
\end{lemma}

From \cref{lemma-optimality-recursion}, there is a clear motivation for \cref{assumption-stability}. Indeed, if we apply \cref{assumption-stability}, \cref{lemma-optimality-recursion} produces the following simple recursive relationship.

\begin{lemma} \label{lemma-optimality-recursion-simple}
If \cref{assumption-flb,assumption-local-holder,assumption-unbiased,assumption-moment,assumption-stability}, and \cref{prop-psd,prop-cond-num} hold, and $\lbrace \theta_k \rbrace$ satisfy \cref{eqn-sgd-update}, then there exists a $K \in \mathbb{N}$ such that for any $j+1 \in \mathbb{N}$ and any $k \geq \min\lbrace K, j+1 \rbrace$, 
\begin{equation}
\begin{aligned}
&\cond{ (F(\theta_{k+1}) - \flb) \1{ \tau_j > k } }{\mathcal{F}_k}  \\
&\leq \left( 1 + \lambda_{\max}(M_k)^{1+\alpha} \frac{C_2}{1+\alpha} \right) (F(\theta_k) - \flb) \1{ \tau_j > k - 1} \\
& - \frac{1}{2} \lambda_{\min}(M_k) \norm{ \dot F (\theta_k) }_2^2 \1{ \tau_j > k - 1} + \lambda_{\max}(M_k)^{1+\alpha}\frac{C_1}{1+\alpha}.
\end{aligned}
\end{equation}
\end{lemma}

Just as in the pseudo-global strategy, \cref{lemma-optimality-recursion-simple} can be combined with standard martingale results \citep[Exercise II.4]{robbins1971,neveu1975} to conclude that the limit of $F(\theta_k)$ exists and is finite on the event $\cup_{j=0}^\infty \lbrace \tau_j = \infty \rbrace$ (see \cref{corollary-stability-union-stop-times}). Also as in the pseudo-global strategy, by improving on the arguments in \cite{patel2020,patel2021stochastic}, we show that $\liminf_k \dot F(\theta_k) = 0$ on the event $\cup_{j=0}^\infty \lbrace \tau_j = \infty \rbrace$ (see \cref{lemma-liminf-gradient}). 

Finally, we show that $\cup_{j=0}^\infty \lbrace \tau_j = \infty \rbrace$ is a probability one event (see \cref{theorem-probability-stop-times}). This statement should not come as a surprise on the event $\lbrace \lim_{k} \theta_k ~\mathrm{exists}\rbrace$, but is slightly surprising that it must also hold on $\lbrace \lim_k \norm{\theta_k}_2 = \infty \rbrace$. By combining these results, we can conclude as follows.

\begin{theorem}[Stability] \label{theorem-stability}
Let $\theta_0$ be arbitrary, and let $\lbrace \theta_k : k \in \mathbb{N} \rbrace$ be defined according to \cref{eqn-sgd-update} with $\lbrace M_k : k+1 \rbrace$ satisfying \cref{prop-psd,prop-max-eig,prop-min-eig,prop-cond-num}. Suppose \cref{assumption-flb,assumption-local-holder,assumption-unbiased,assumption-moment,assumption-stability} hold. Then, 
\begin{enumerate}[noitemsep,parsep=0em,topsep=0em]
\item There exists a finite random variable, $F_{\lim}$, such that $\lim_{k \to \infty} F(\theta_k) = F_{\lim}$ with probability one;
\item $\liminf_{k \to \infty} \dot F(\theta_k) = 0$ with probability one.
\end{enumerate}
\end{theorem}
\begin{proof}
Using \cref{corollary-stability-union-stop-times}, we conclude that $\exists F_{\lim}$ that is finite such that $\lim_k F(\theta_k) = F_{\lim}$ on  $\cup_{j=0}^\infty \lbrace \tau_j = \infty \rbrace$. Using \cref{lemma-liminf-gradient}, we conclude that $\liminf_k \dot F(\theta_k) = 0$ on $\cup_{j=0}^\infty \lbrace \tau_j = \infty \rbrace$. Finally, we apply \cref{theorem-probability-stop-times} to conclude that $\inlinePrb{\cup_{j=0}^\infty \lbrace \tau_j = \infty \rbrace} = 1.$
\end{proof}

We would like to demonstrate a simple example of how we would use \cref{theorem-stability}. Consider applying SGD to linear regression as specified in \cref{subsection-lr}. For this example, it is straightforward to verify that the assumptions of \cref{theorem-stability} are satisfied. Therefore, if we are to apply SGD to linear regression, we know that with probability one, $F(\theta_k) \to F_{\lim}$ which is finite. Since $F(\theta) \to \infty$ as $\theta \to \infty$, we know that $\lbrace \theta_k \rbrace$ cannot diverge. Hence, $\lbrace \theta_k \rbrace$ must remain finite with probability one. By \cref{theorem-global-convergence}, $\lbrace \theta_k \rbrace$ must converge to a stationary point. Since this stationary point is unique in our specific example of linear regression, we know that SGD must converge to the global minimizer of the linear regression problem. Note, we can follow this outline to draw similar conclusions in more complex situations.

%% file: section/conclusion.tex
In this work, we studied the global convergence analysis of Stochastic Gradient Descent with diminishing step size. We began our discussion by producing three simple problems for which the common assumptions (i.e., global H\"{o}lder continuity, bounded variance) and their generalizations in the SGD literature are violated. Indeed, to our knowledge, there does not exist theory that covers these problems. For example, prior to our work, it was unknown what SGD with arbitrary initialization and diminishing step sizes will do on simple neural network problems, which raised the question of what SGD is doing on more complicated learning problems. 

Motivated by our example problems, we considered a more general set of assumptions (see \cref{assumption-local-holder,assumption-moment}). Given the generality of our assumptions, we developed a new analysis technique that is of interest beyond this work, and we proved that SGD's iterates either converge to a stationary point or diverge. Thus, we now know how SGD with arbitrary initialization and diminishing step sizes will behave on a much larger class of learning problems.

We note that we do not provide rate of convergence results mainly because it is \textit{impossible} for the broad class of functions admitted by our assumptions \citep{wolpert1997}. We stress that global rates of convergence (e.g., complexity statements) results that exist do not apply to the two simple neural network problems that we supplied at the beginning of this work. 

We also studied what happens when SGD's iterates diverge. To this end, we required an additional assumption under which we developed another novel analysis technique and showed that, regardless of SGD's iterates' behavior, the objective function will converge to a finite random variable with probability one. Unfortunately, we make an assumption that we were not able to interpret, but we will leave this to future work.

%% file: section/examples-details.tex
\subsection{Simple Linear Regression} \label{subsection-lr}
We begin our exploration of assumptions with a rather simple problem. Let $Z \in \mathbb{R}$ be a random variable such that $\E{ Z^2} = 1$ and $\E{Z^4} = 2$. Moreover, let $\epsilon$ be an independent random variable with mean zero and variance $1$. Finally, let $\theta^* \in \mathbb{R}$ and define $Y = Z \theta^* + \epsilon$. Consider the estimation problem of minimizing $F(\theta)$ where
\begin{equation}
F(\theta) = \frac{1}{2} \E{ (Z \theta - Y)^2} = \frac{1}{2} (\theta - \theta^*)^2 + \frac{1}{2}.
\end{equation}

Letting $X = (Y,Z)$, let $f(\theta,X) = 0.5 (Z \theta - Y)^2$. Now, the variance of $\dot f(\theta,X)$ is 
\begin{align}
&\E{ ( (Z^2 -1)(\theta - \theta^*) - Z\epsilon )^2} \nonumber\\
&= \E{ (Z^2 - 1)^2}(\theta - \theta^*)^2 - 2(\theta - \theta^*) \E{ (Z^2 -1)Z \epsilon} + \E{ Z^2 \epsilon^2} \\
&= (\theta - \theta^*)^2 + 1.
\end{align}

Clearly, the variance scales with the error in the parameter, which violates the common bounded noise model assumption. In particular, as $|\theta| \to \infty$, the variance diverges.

On the other hand, the simple linear regression problem does satisfy our assumptions. In particular,
\begin{enumerate}
\item \cref{assumption-flb,assumption-unbiased} are easily verified.
\item Given that $\dot F$ is globally Lipschitz continuous, it is locally Lipschitz continuous. Therefore, \cref{assumption-local-holder} is satisfied.
\item From the variance calculation of $\dot f(\theta,X)$, we conclude
\begin{equation}
\E{ \dot f(\theta,X)^2} = 2(\theta - \theta^*)^2 + 1, 
\end{equation}
which is a continuous function. Hence, \cref{assumption-moment} is satisfied.
\end{enumerate}

\subsection{Feed Forward Network for Binary Classification} \label{subsection-ffn}

We now prove \cref{proposition-ffn}. Consider the binary classification problem with label $Y$ and feature $Z$ where $(Y,Z) = (0,0)$ with probability $1/2$ and $(Y,Z) = (1,1)$ with probability $1/2$. We solve this classification problem using the network shown in \cref{figure-ffn} with $\sigma$ linear and $\varphi$ sigmoid. We will train this model using the binary cross entropy loss function. Letting $X = (Y,Z)$ and $\theta = (W_1,W_2,W_3,W_4)$, 
\begin{equation}
f(\theta,X) = -Y\log( \hat y) - (1-Y) \log( 1 - \hat y) + \frac{1}{2}\sum_{i=1}^4 W_i^2,
\end{equation}
and
\begin{equation}
\hat y = \begin{cases}
\frac{1}{2} & Z=0 \\
\frac{1}{1 + \exp(-W_4W_3 W_2W_1)} & Z=1.
\end{cases}
\end{equation}

From this, we compute,
\begin{equation}
F(\theta) = \frac{1}{2}\log(2) + \frac{1}{2}\log[1 + \exp(-W_4 W_3 W_2 W_1)] + \frac{1}{2} \sum_{i=1}^4 W_i^2.
\end{equation}
Moreover,
\begin{equation}
\dot f(\theta,X) = \begin{cases}
 \begin{bmatrix}
W_1 \\
W_2 \\
W_3 \\
W_4
\end{bmatrix} & (Y,Z) = (0,0) \\
\frac{-1}{1 + \exp(W_4 W_3 W_2 W_1)} \begin{bmatrix}
W_4 W_3 W_2 \\
W_4 W_3 W_1 \\
W_4 W_2 W_1 \\
W_3 W_2 W_1
\end{bmatrix} + \begin{bmatrix}
W_1 \\
W_2 \\
W_3 \\
W_4
\end{bmatrix} & (Y,Z) = (1,1), \\
\end{cases}
\end{equation}
and, consequently,
\begin{equation}
\dot F(\theta) = \frac{-1/2}{1 + \exp(W_4 W_3 W_2 W_1)}\begin{bmatrix}
W_4 W_3 W_2 \\
W_4 W_3 W_1 \\
W_4 W_2 W_1 \\
W_3 W_2 W_1
\end{bmatrix} + \begin{bmatrix}
W_1 \\
W_2 \\
W_3 \\
W_4
\end{bmatrix}.
\end{equation}

Finally, letting $\ddot F(\theta) = \nabla^2 F(\psi) \vert_{\psi=\theta}$,
\begin{equation}
\begin{aligned}
\ddot{F}(\theta) &= \frac{-0.5}{1 + \exp(W_4W_3 W_2 W_1)} \begin{bmatrix}
0 & W_4 W_3 & W_4 W_2 & W_3 W_2 \\
W_4 W_3 & 0 & W_4 W_1 & W_3 W_1 \\
W_4 W_2 & W_4 W_1 & 0 & W_2 W_1 \\
W_3 W_2 & W_3 W_1 & W_2 W_1 & 0
\end{bmatrix} \\
&+
\frac{0.5 \exp(W_4 W_3 W_2 W_1)}{[ 1 + \exp(W_4 W_3 W_2 W_1)]^2} \begin{bmatrix}
W_4 W_3 W_2 \\
W_4 W_3 W_1 \\
W_4 W_2 W_1 \\
W_3 W_2 W_1
\end{bmatrix} \begin{bmatrix}
W_4 W_3 W_2 \\
W_4 W_3 W_1 \\
W_4 W_2 W_1 \\
W_3 W_2 W_1
\end{bmatrix}' + I_4,
\end{aligned}
\end{equation}
where $I_4$ is the $4 \times 4$ identity matrix.

We first establish that $\dot{F}(\theta)$ is not globally Lipschitz continuous. With $\theta = (1,-1,W_3,W_3)$ and $\phi = (1,-1,W_3,0)$, it is enough to find a lower bound for the first component of $\dot F(\theta) - \dot F(\phi)$, denoted by $\dot F_1(\theta) - \dot F_1(\phi)$. To this end,
\begin{align}
| \dot F_1(\theta) - \dot F_1(\phi) | = \frac{0.5 W_3^2}{1 + \exp(-W_3^2)} \geq \frac{1}{4}| W_3 - 0|^2.
\end{align}
Thus, $\dot F$ is not globally Lipschitz. 

We now establish that $F$ does not satisfy $(L_0,L_1)$-smoothness. That is, we show that there is no $L_0, L_1 \geq 0$ such that $\inlinenorm{ \ddot F(\theta)} \leq L_0 \inlinenorm{ \dot F(\theta)} + L_1$, where the norms can be chosen arbitrarily owing to the equivalence of norms in finite-dimensional vector spaces. To see this, note that the Frobenius norm of $\ddot F(\theta)$ is lower bounded by the absolute value of the $[1,1]$ entry. Using notation,

\begin{equation}
\frac{0.5 \exp(W_4 W_3 W_2 W_1)}{[ 1 + \exp(W_4 W_3 W_2 W_1)]^2} (W_4 W_3 W_2)^2 + 1 \leq \norm{ \ddot{F}(\theta)}_F.
\end{equation}

Let $\theta = (0, W_4, W_4, W_4)$, then the lower bound is
\begin{equation}
\frac{1}{8} W_4^6 \leq \norm{ \ddot{F}(\theta)}_F.
\end{equation}

Notice, for this same choice of $\theta$, the $l^1$ norm of the gradient is bounded above by
\begin{equation}
\norm{ \dot F(\theta)}_1 \leq \frac{1}{4} |W_4|^3 + 3|W_4|.
\end{equation}

For any choice of $L_0, L_1 > 0$, we conclude that there is a $W_4$ sufficiently large such that, for this parametrization of $\theta$,
\begin{equation}
L_0 \norm{\dot F(\theta)} + L_1 \leq L_0 [\frac{1}{4} |W_4|^3 + 3|W_4|] + L_1 < \frac{1}{8} W_4^6 \leq \norm{ \ddot{F}(\theta)}_F.
\end{equation}
Thus, we see that no $L_0$ nor $L_1$ can exist that will satisfy the $(L_0,L_1)$-smooth assumption for all choices of $\theta$.

To show that the variance is not bounded, we study the variance of the first component of $\dot f(\theta,X)$ which we denote by $\dot f_1(\theta,X)$. By direct calculation,
\begin{equation}
\E{ ( \dot f_1(\theta,X) - \dot F_1(\theta) )^2} = \frac{1}{4} \frac{W_4^2 W_3^2 W_2^2}{[1 + \exp(W_4 W_3 W_2 W_1)]^2}.
\end{equation}
We again consider $\theta = (1,-1,W_3,W_3)$, then the variance at this value of $\theta$ is
\begin{equation}
\frac{1}{4} \frac{W_3^4}{[ 1 + \exp(-W_3^2)]^2} \geq \frac{1}{16}W_3^4.
\end{equation}
Therefore, as $W_3 \to \infty$, the variance goes to infinity. That is, the variance of the stochastic gradients is unbounded. 

On the other hand, the problem does satisfy our assumptions. In particular,
\begin{enumerate}
\item \cref{assumption-flb,assumption-unbiased} are easily verified.
\item Given that $\dot F$ is continuously differentiable, then compactness and continuity of the derivative of $\dot F$ imply that it is locally Lipschitz continuous. Therefore, \cref{assumption-local-holder} is satisfied.
\item Given the computation of the variance for the first component, we have $\inlineE{ \dot f_1(\theta,X)^2 }$ is
\begin{equation}
\frac{1}{4} \frac{W_4^2 W_3^2 W_2^2}{[1 + \exp(W_4 W_3 W_2 W_1)]^2} + \dot F_1(\theta)^2,
\end{equation}
which is a continuous function. By repeating this argument for each component, we conclude that \cref{assumption-moment} is satisfied.
\end{enumerate}

\subsection{Recurrent Neural Network for Binary Classification} \label{subsection-rnn}
Consider observing one of two sequences $(1,0,0,0)$ or $(0,0,0,0)$ with equal probabilities, and suppose that each sequence corresponds to the label $1$ or $0$, respectively. Now consider \cref{figure-rnn} to be a 1-dimensional linear recurrent neural network which reads each element of the sequence and uses a logistic output layer to predict either a label of one or zero. If we fix $H_0=0$ and $W_3 =1$, then the model predicts the probability of a $1$ label as
\begin{equation}
\hat{y}(Z_0,Z_1,Z_2,Z_3) = \frac{\exp(W_1^3 W_2 Z_0)}{1 + \exp(W_1^3 W_2 Z_0)}.
\end{equation}

If we use the binary cross entropy loss with $\ell^2$ regularization, and let $X = (Y,Z_0,Z_1,Z_2,Z_3)$ and $\theta = (W_1,W_2)$  then
\begin{align}
f(\theta,X) 
&= - Y \log \hat{y}(Z_0,Z_1,Z_2,Z_3) - (1 - Y) \log [1 - \hat{y}(Z_0,Z_1,Z_2,Z_3) ] + \frac{1}{2}(W_1^2 + W_2^2)\\
&= -Y \left[ W_1^3W_2Z_0 - \log( 1 + \exp(W_1^3W_2 Z_0) )\right] + (1 - Y)\log( 1 + \exp(W_1^3W_2 Z_0 )) \nonumber \\
&\quad + \frac{1}{2}(W_1^2 + W_2^2) \\
&= - W_1^3W_2 Z_0 Y + \log(1 + \exp(W_1^3W_2 Z_0 )) + \frac{1}{2}(W_1^2+W_2^2),
\end{align}
and
\begin{equation}
\dot f(\theta,X) = \begin{bmatrix}
-3W_1^2 W_2 Z_0 Y + \frac{3W_1^2 W_2 Z_0 \exp( W_1^3 W_2 Z_0)}{1 + \exp(W_1^3 W_2 Z_0)} + W_1 \\
-W_1^3 Z_0 Y + \frac{W_1^3 Z_0 \exp( W_1^3 W_2 Z_0)}{1 + \exp(W_1^3 W_2 Z_0)} + W_2
\end{bmatrix}
\end{equation}

Taking the expectations, we compute
\begin{equation}
F(\theta) = \frac{1}{2}\left[ \log(2) + \log( 1 + \exp(W_1^3W_2) ) - W_1^3W_2 + W_1^2 + W_2^2 \right],
\end{equation}
and
\begin{equation}
\dot F(\theta) = \begin{bmatrix}
\frac{-3W_1^2W_2}{2} \frac{1}{1 + \exp(W_1^3 W_2)} + W_1 \\
\frac{ - W_1^3}{2} \frac{1}{1 + \exp(W_1^3 W_2)} + W_2
\end{bmatrix}.
\end{equation}

Taking another derivative and letting $\ddot F(\theta) = \nabla^2 F(\psi) \vert_{\psi = \theta}$, 
\begin{equation}
\ddot F(\theta) = \begin{bmatrix}
\frac{9W_1^4 W_2^2 \exp(W_1^3W_2)}{2(1+\exp(W_1^3W_2))^2} - \frac{3W_1 W_2}{1+\exp(W_1^3 W_2)}+1 & \frac{3W_1^5 W_2 \exp(W_1^3W_2)}{2(1+\exp(W_1^3 W_2))^2} - \frac{3W_1^2}{2}\frac{1}{1+\exp(W_1^3 W_2)} \\
\frac{3W_1^5 W_2 \exp(W_1^3 W_2)}{2(1+\exp(W_1^3W_2))^2} - \frac{3 W_1^2}{2}\frac{1}{1 + \exp(W_1^3 W_2)} & \frac{W_1^6 \exp(W_1^3 W_2)}{2(1+\exp(W_1^3 W_2))^2} + 1
\end{bmatrix}.
\end{equation}

We first establish that $\dot F$ is not globally Lipschitz continuous. Notice, if we set $W_2 = 1$, then the first and second component of $\dot F(\theta)$ are proportional to $-W_1^2$ and $-W_1^3$ respectively, which are not globally Lipschitz continuous functions. 

We now show that $F$ also does not satisfy $(L_0,L_1)$-smoothness. Notice that, using the bottom right entry of $\ddot F(\theta)$, 
\begin{equation}
\frac{W_1^6 \exp(W_1^3 W_2)}{2(1+\exp(W_1^3 W_2))^2} < \norm{ \ddot F(\theta)}_F,
\end{equation}
and
\begin{equation}
\norm{ \dot F(\theta)}_1 \leq \frac{3W_1^2 |W_2| + |W_1|^3}{2[1 + \exp(W_1^3 W_2)]} + |W_1| + |W_2|.
\end{equation}
If we choose $W_2 = 0$, then, for any $L_0, L_1 \geq 0$ there exists a $|W_1|$ sufficiently large such that
\begin{equation}
\frac{W_1^6}{8} < \norm{ \ddot F(\theta)}_F \not\leq L_0 \norm{ \dot F(\theta)}_1 + L_1 \leq L_0 \left( \frac{|W_1|^3}{4} + |W_1|\right) + L_1.
\end{equation}
Hence, $F(\theta)$ is not $(L_0, L_1)$-smooth.

Moreover, computing the trace of the variance of $\dot f(\theta,X)$, we recover
\begin{equation}
\E{ \norm{ \dot f(\theta,X) - \dot F(\theta) }_2^2} = \left( \frac{3W_1^2W_2}{2[ 1 + \exp(W_1^3 W_2)]}  \right)^2 + \left( \frac{W_1^3}{2[1+\exp(W_1^3W_2)]} \right)^2, 
\end{equation}
which does not satisfy a bounded variance assumption (choose $W_2 = 0$ and let $W_1 \to \infty$). Thus, any work that makes either a global Lipschitz bound on the gradient or a global noise model bound  fails to apply to this simple recurrent neural network training problem. 

On the other hand, the problem does satisfy our assumptions. In particular,
\begin{enumerate}
\item \cref{assumption-flb,assumption-unbiased} are easily verified.
\item Given that $\dot F$ is continuously differentiable, then compactness and continuity of the derivative of $\dot F$ imply that it is locally Lipschitz continuous. Therefore, \cref{assumption-local-holder} is satisfied.
\item Moreover,  
\begin{equation}
\E{ \norm{\dot f(\theta,X)}_2^2} = \left( \frac{3W_1^2W_2}{2[ 1 + \exp(W_1^3 W_2)]}  \right)^2 + \left( \frac{W_1^3}{2[1+\exp(W_1^3W_2)]} \right)^2 + \norm{\dot F(\theta)}_2^2, 
\end{equation}
which is a continuous function. Hence, \cref{assumption-moment} is satisfied.
\end{enumerate}

\subsection{Poisson Regression} \label{subsection-poisson}

Here, we consider the task of estimating a Poisson regression model for data $X=(Y,Z)$ where $Y$ is a count response variable and $Z$ is the covariate. To make this problem simpler, we will assume that both $Y$ and $Z$ are independent Poisson random variables with mean $1$, which implies that the parameter in the model, $\theta^* = 0$. If we use a likelihood framework, then, up to a constant depending on $Y$,
\begin{equation}
f(\theta,X) = -YZ\theta + \exp(\theta Z),
\end{equation}
and
\begin{equation}
\dot f(\theta,X) = -YZ + Z\exp(\theta Z).
\end{equation}

From this, we compute
\begin{equation}
F(\theta) = -\theta + \exp( \exp(\theta) - 1),
\end{equation}
\begin{equation}
\dot F(\theta) = -1 + \exp( \exp(\theta) + \theta - 1),
\end{equation}
and, letting $\nabla^2 F(\psi) |_{\psi=\theta} = \ddot F(\theta)$, 
\begin{equation}
\ddot F(\theta) = (\exp(\theta) + 1) \exp( \exp(\theta) + \theta - 1).
\end{equation}

We begin by showing that $\dot F(\theta)$ is not globally Lipschitz continuous. To do so, for any $\theta > 0$,  note
\begin{equation}
| \dot F(\theta) - \dot F(0) | = \exp( \exp(\theta) + \theta - 1) - 1 > \exp(\theta) - 1  \geq \theta + \theta^2/2.
\end{equation}
Thus, for any $L > 0$ there exists a $\theta > 0$ such that $|\dot F(\theta) - \dot F(0)| > L |\theta|$.

We now show that $F(\theta)$ does not satisfy the $(L_0,L_1)$-smooth assumption. Note, for $\theta \geq 0$,
\begin{equation}
\exp( \exp(\theta) + 2 \theta - 1) < \ddot{F}(\theta),
\end{equation}
and
\begin{equation} \label{eqn-poisson-dotf-upper}
\dot F(\theta) < \exp(\exp(\theta) + \theta - 1).
\end{equation}
It follows that for any $L_0,L_1 > 0$, there exists a $\theta > 0$ such that $L_0 |\dot F(\theta)| + L_1 < \ddot F(\theta).$

For the noise, we compute the second moment of $\dot f(\theta,X)$. That is,
\begin{align}
\E{ \dot f(\theta,X)^2 } &= \E{ Y^2 Z^2 - 2 Y Z^2 \exp(\theta Z) + Z^2 \exp(2 \theta Z) } \\
						&= 4 - 2 \E{ Z^2 \exp(\theta Z)} + \E{ Z^2 \exp(2 \theta Z) } \\
						&= 4 - 2(\exp(\theta) +1)\exp(\exp(\theta) + \theta -1) \nonumber \\
						& + (\exp(2\theta) + 1) \exp(\exp(2\theta) + 2\theta - 1). 
\end{align}

It is clear from this calculation that the variance (computed by subtracting off $\dot F(\theta)^2$) will diverge as $\theta$ tends to infinity. To show that \cite[Assumption 4.3c]{bottou2018} does not apply, it is enough to show that its generalization, \cite[Assumption 2]{khaled2020} does not apply. To this end, we must show that there does not exists a $C_0, C_1, C_2 \geq 0$ such that, $\forall \theta$,
\begin{equation}
\E{ \dot f(\theta,X)^2} \leq C_0 + C_1 F(\theta) + C_2 |\dot F(\theta)|^2.
\end{equation}
From our calculations, it is easy to verify that $F(\theta)$ and $\dot F(\theta)$ are dominated by $\exp(2\exp(\theta))$, and that the second moment of the stochastic gradient is bounded from below by $\exp(\exp(2\theta))$ for $\theta \geq \log(4)$. Hence, for any $C_0,C_1,C_2 \geq 0$, there exists $\theta$ sufficiently large such that 
\begin{equation}
C_0 + C_1 F(\theta) + C_2 |\dot F(\theta)|^2 \leq C_0 + (C_1 + C_2) \exp(2 \exp(\theta)) < \exp(\exp(2\theta)) \leq \E{ \dot f(\theta,X)^2}.
\end{equation}
Thus, \cite[Assumption 4.3c]{bottou2018} and \cite[Assumption 2]{khaled2020} do not hold.

On the other hand, the problem does satisfy our assumptions. In particular,
\begin{enumerate}
\item \cref{assumption-flb,assumption-unbiased} are easily verified.
\item Given that $\ddot F$ is continuous, \cref{assumption-local-holder} is satisfied.
\item Moreover, we can use the calculated value $\inlineE{ \dot f(\theta,X)^2}$, which is a continuous function, as $G(\theta)$ to satisfy \cref{assumption-moment}.
\end{enumerate}

\subsection{Noiseless Feed Forward Network for Binary Classification}

Out of interest, we reconsider the second example but construct a different data distribution that produces noiseless stochastic gradient. Consider the binary classification problem with label $Y$ and feature $Z$ where $(Y,Z) = (0,-1)$ with probability $1/2$ and $(Y,Z) = (1,1)$ with probability $1/2$. We solve this classification problem using the network shown in \cref{figure-ffn} with $\sigma$ linear and $\varphi$ sigmoid. We will train this model using the binary cross entropy loss function. Letting $X = (Y,Z)$ and $\theta = (W_1,W_2,W_3,W_4)$, 
\begin{equation}
f(\theta,X) = -Y\log( \hat y) - (1-Y) \log( 1 - \hat y) + \frac{1}{2}\sum_{i=1}^4 W_i^2,
\end{equation}
and
\begin{equation}
\hat y = \frac{1}{1 + \exp(-W_4W_3 W_2W_1 Z)}.
\end{equation}

Moreover,
\begin{equation}
\dot f(\theta,X) = Z (\hat{y} - Y) \begin{bmatrix}
W_4 W_3 W_2 \\
W_4 W_3 W_1 \\
W_4 W_2 W_1 \\
W_3 W_2 W_1
\end{bmatrix} + \begin{bmatrix}
W_1 \\
W_2 \\
W_3 \\
W_4
\end{bmatrix},
\end{equation}
and, consequently,
\begin{equation}
\dot F(\theta) = \frac{-1}{1 + \exp(W_4 W_3 W_2 W_1)}\begin{bmatrix}
W_4 W_3 W_2 \\
W_4 W_3 W_1 \\
W_4 W_2 W_1 \\
W_3 W_2 W_1
\end{bmatrix} + \begin{bmatrix}
W_1 \\
W_2 \\
W_3 \\
W_4
\end{bmatrix}.
\end{equation}

We first establish that $\dot{F}(\theta)$ is not globally Lipschitz continuous. With $\theta = (1,-1,W_3,W_3)$ and $\phi = (1,0,0,0)$, it is enough to find a lower bound for the first component of $\dot F(\theta) - \dot F(\phi)$, denoted by $\dot F_1(\theta) - \dot F_1(\phi)$. To this end,
\begin{align}
| \dot F_1(\theta) - \dot F_1(\phi) | = \frac{W_3^2}{1 + \exp(-W_3^2)} \geq \frac{1}{2}| W_3 - 0|^2.
\end{align}
Thus, $\dot F$ is not globally Lipschitz. 

On the other hand, the problem does satisfy our assumptions. In particular,
\begin{enumerate}
\item \cref{assumption-flb,assumption-unbiased} are easily verified.
\item Given that $\dot F$ is continuously differentiable, then compactness and continuity of the derivative of $\dot F$ imply that it is locally Lipschitz continuous. Therefore, \cref{assumption-local-holder} is satisfied.
\item Moreover, $\dot f(\theta,Z) = \dot F(\theta)$---that is, there $\dot f(\theta,Z)$ has zero variance for the distribution that we have constructed. Therefore, 
\begin{equation}
\E{ \norm{\dot f(\theta,X)}_2^2} = \norm{\dot F(\theta)}_2^2, 
\end{equation}
which is a continuous function. Hence, \cref{assumption-moment} is satisfied.
\end{enumerate}

%% file: section/technical_lemmas.tex
\begin{lemma}[\cref{lemma-taylor-holder-bound}]
Suppose \cref{assumption-flb,assumption-local-holder} hold. Then, for any $\theta, \varphi \in \mathbb{R}^p$,
\begin{equation}
F(\varphi) - \flb \leq F(\theta) - \flb + \dot F(\theta)'(\varphi - \theta) + \frac{L(\theta, \varphi)}{1+\alpha} \norm{ \varphi - \theta}_2^{1+\alpha}.
\end{equation}
\end{lemma}
\begin{proof}
By Taylor's theorem, 
\begin{equation}
F(\varphi) - \flb = F(\theta) - \flb + \int_0^1 \dot F( \theta + t(\varphi - \theta) )'(\varphi - \theta) dt.
\end{equation}
Now, add and subtract $\dot F(\theta)$ to $\dot F(\theta + t(\varphi - \theta))$ in the integral, then apply \cref{assumption-local-holder}. By \cref{def-local-holder-constant},
\begin{equation}
\norm{ \dot F(\theta + t(\varphi - \theta)) - \dot F(\theta) }_2 \leq L(\theta, \varphi) t^\alpha \norm{ \theta - \varphi}_2^\alpha.
\end{equation}
We conclude,
\begin{equation}
F(\varphi) - \flb \leq F(\theta) - \flb + \dot F(\theta)' (\varphi - \theta) + L(\theta,\varphi) \norm{ \varphi - \theta }_2^{1+\alpha}\int_0^1  t^\alpha dt.
\end{equation}
By computing the integral, the result follows.
\end{proof}

\begin{lemma} \label{lemma-lr-eig-bound}
Suppose $\lbrace M_k : k +1 \in \mathbb{N} \rbrace$ satisfy \cref{prop-psd,prop-cond-num}. Then $\forall C > 0$, $\exists K \in \mathbb{N}$ such that $\forall k \geq K$,
\begin{equation}
\lambda_{\min}( M_k) - \frac{C}{2} \lambda_{\max}(M_k)^{1+\alpha} \geq \frac{1}{2} \lambda_{\min}(M_k).
\end{equation}
\end{lemma}
\begin{proof}
Fix $C > 0$. Rearranging the conclusion, we see that it is equivalent to prove that $\exists K \in \mathbb{N}$ such that $\forall k \geq K$, $1/C \geq \lambda_{\max}(M_k)^\alpha \kappa(M_k)$. This follows from \cref{prop-cond-num}.
\end{proof}

\begin{lemma} \label{lemma-taylor-lower-bound}
For any $\theta \in \mathbb{R}^p$, $v \in \mathbb{R}$, $L > 0$ and $\alpha \in (0,1]$,
\begin{equation}
\frac{L}{1+\alpha} v^{1+\alpha} - \norm{ \dot F(\theta) }_2 v \geq -\frac{\alpha}{1+\alpha} \left[ \frac{\norm{ \dot F(\theta)}_2^{1+\alpha}}{L}  \right]^{1/\alpha}.
\end{equation}
\end{lemma}
\begin{proof}
If we minimize the left hand side of the inequality, we see that a minimum value occurs when $v^{\alpha} = \norm{ \dot F(\theta) }_2 / L \geq 0$. Solving for $v$ and plugging this back into the left hand side, we conclude that the inequality holds.
\end{proof}

%% file: section/global-convergence-analysis.tex
We begin by first deriving a recursive relationship between the optimality gap at iteration $k+1$ and the optimality gap at iteration $k$ on the events $\lbrace \mathcal{B}_j(R) \rbrace$ as defined in \cref{eqn-iterates-in-ball} for arbitrary $R \geq 0$. Using this result, we then provide an analysis of the convergence of the objective function. Then, we turn our attention to the gradient function. Note, $B(\theta,r)$ is the open ball around $\theta$ of radius $r$.

\subsection{A Recursive Relationship}

\begin{lemma}[\cref{lemma-recursion-constrained}]
Let $\lbrace M_k \rbrace$ satisfy \cref{prop-psd}.
Suppose \cref{assumption-flb,assumption-local-holder,assumption-unbiased,assumption-moment} hold. 
Let $\lbrace \theta_k \rbrace$ satisfy \cref{eqn-sgd-update}.
Then, $\forall R \geq 0$,
\begin{equation}
\begin{aligned}
&\cond{ [F(\theta_{k+1}) - \flb] \1{ \mathcal{B}_{k+1}(R) } }{\mathcal{F}_k} \leq [F(\theta_k) - \flb] \1{ \mathcal{B}_{k}(R) } \\
&\quad - \lambda_{\min}(M_k) \norm{ \dot{F}(\theta_k)}_2^2 \1{\mathcal{B}_k(R)} + \frac{L_{R+1} + \partial F_R}{1+\alpha} \lambda_{\max}(M_k)^{1+\alpha} G_R,
\end{aligned}
\end{equation}
where $G_R = \sup_{\theta \in \overline{B(0,R)}} G(\theta) < \infty$ with $G(\theta)$; and $\partial F_R  = \sup_{ \theta \in \overline{B(0,R)} } \inlinenorm{ \dot{F}(\theta)}_2 (1+\alpha) < \infty$.
\end{lemma}

\begin{proof}
Fix $R \geq 0$. For any $k +1 \in \mathbb{N}$, the definition of local H\"{o}lder continuity implies that $L_{R+1}$ is well defined (see \cref{def-local-holder-constant}). Therefore, \cref{lemma-taylor-holder-bound} implies
\begin{equation}
\begin{aligned}
&[F(\theta_{k+1}) - \flb] \1{ \mathcal{B}_{k+1}(R+1) } \\
&\quad \leq \left( [F(\theta_k) - \flb] + \dot{F}(\theta_k)'(\theta_{k+1} - \theta_k) + \frac{L_{R+1}}{1+\alpha} \norm{ \theta_{k+1} - \theta_k }_2^{1+\alpha} \right) \1{ \mathcal{B}_{k+1}(R+1) }.
\end{aligned}
\end{equation}

Now, since $\overline{B(0,R)} \subset \overline{B(0,R+1)}$, it also holds true that
\begin{equation}
\begin{aligned}
&[F(\theta_{k+1}) - \flb] \1{ \mathcal{B}_{k+1}(R) } \\
&\quad \leq \left( [F(\theta_k) - \flb] + \dot{F}(\theta_k)'(\theta_{k+1} - \theta_k) + \frac{L_{R+1}}{1+\alpha} \norm{ \theta_{k+1} - \theta_k }_2^{1+\alpha} \right) \1{ \mathcal{B}_{k+1}(R) }.
\end{aligned}
\end{equation}

Our goal now is to replace $\mathcal{B}_{k+1}(R)$ on the right hand side by $\mathcal{B}_{k}(R)$. However, there is a technical difficulty which we must address. First, it follows from the preceding inequality that
\begin{equation} \label{ineq-30sx4}
\begin{aligned}
&[F(\theta_{k+1}) - \flb] \1{ \mathcal{B}_{k+1}(R) } \\
& \leq \left( [F(\theta_k) - \flb] + \dot{F}(\theta_k)'(\theta_{k+1} - \theta_k) + \frac{L_{R+1}}{1+\alpha} \norm{ \theta_{k+1} - \theta_k }_2^{1+\alpha} \right) \\
&\quad \times \bigg{(} \1{ \mathcal{B}_{k+1}(R) } - \1{\mathcal{B}_k(R)} \bigg{)} \\
&\quad + \left( [F(\theta_k) - \flb] + \dot{F}(\theta_k)'(\theta_{k+1} - \theta_k) + \frac{L_{R+1}}{1+\alpha} \norm{ \theta_{k+1} - \theta_k }_2^{1+\alpha} \right) \1{ \mathcal{B}_{k}(R) }.
\end{aligned}
\end{equation}

The first term on the right hand side of the inequality only contributes meaningfully if it is positive. Since $\1{ \mathcal{B}_{k}(R) } \geq \1{ \mathcal{B}_{k+1}(R) }$, then two statements hold: (i) $\1{\mathcal{B}_k(R)}\1{\mathcal{B}_{k+1}(R)} = \1{\mathcal{B}_{k+1}(R)}$; and (ii) the first term of the right hand side of \cref{ineq-30sx4} is positive if and only if 
\begin{equation} \label{event-d49fa}
\left([F(\theta_k) - \flb] + \dot{F}(\theta_k)'(\theta_{k+1} - \theta_k) + \frac{L_{R+1}}{1+\alpha} \norm{ \theta_{k+1} - \theta_k }_2^{1+\alpha}\right) \1{\mathcal{B}_k(R)} < 0.
\end{equation}
By the choice of $L_{R+1}$, \cref{assumption-flb} and \cref{lemma-taylor-holder-bound} imply that if \cref{event-d49fa} occurs, then $\inlinenorm{\theta_{k+1}}_2 > R + 1 \geq \inlinenorm{\theta_k}_2 + 1$. By the reverse triangle inequality and \cref{eqn-sgd-update}, if \cref{event-d49fa} occurs, then $\inlinenorm{ M_k \dot{f}(\theta_k, X_{k+1})}_2 \geq 1$. Hence,
\begin{equation}
\begin{aligned}
&\left( [F(\theta_k) - \flb] + \dot{F}(\theta_k)'(\theta_{k+1} - \theta_k) + \frac{L_{R+1}}{1+\alpha} \norm{ \theta_{k+1} - \theta_k }_2^{1+\alpha} \right) \\
&\quad \times \bigg{(} \1{ \mathcal{B}_{k+1}(R) } - \1{\mathcal{B}_k(R)} \bigg{)} \\
&\leq \left( -[F(\theta_k) - \flb] - \dot{F}(\theta_k)'(\theta_{k+1} - \theta_k) - \frac{L_{R+1}}{1+\alpha} \norm{ \theta_{k+1} - \theta_k }_2^{1+\alpha} \right) \\
&\quad \times \bigg{(} \1{ \mathcal{B}_{k}(R) } - \1{\mathcal{B}_{k+1}(R)} \bigg{)} \1{ \mathcal{B}_k(R)} \1{ \norm{M_k \dot f (\theta_k, X_{k+1} )}_2 \geq 1 }.
\end{aligned}
\end{equation}

We now compute another coarse upper bound for this inequality. Note, by \cref{assumption-flb} and Cauchy-Schwarz, 
\begin{align}
&\begin{aligned}
&\left( -[F(\theta_k) - \flb] - \dot{F}(\theta_k)'(\theta_{k+1} - \theta_k) - \frac{L_{R+1}}{1+\alpha} \norm{ \theta_{k+1} - \theta_k }_2^{1+\alpha} \right) \\
&\quad \times \bigg{(} \1{ \mathcal{B}_{k}(R) } - \1{\mathcal{B}_{k+1}(R)} \bigg{)} \1{ \mathcal{B}_k(R)} \1{ \norm{M_k \dot f (\theta_k, X_{k+1} )} _2 \geq 1 }
\end{aligned} \\
&\quad \leq \norm{ \dot F(\theta_k) }_2 \norm{ M_k \dot f (\theta_k, X_{k+1} ) }_2 \1{ \mathcal{B}_k(R)} \1{ \norm{M_k \dot f (\theta_k, X_{k+1} )}_2 \geq 1 } \\
&\quad \leq \norm{ \dot F(\theta_k) }_2 \norm{ M_k \dot f (\theta_k, X_{k+1} ) }_2^{1+\alpha}\1{ \mathcal{B}_k(R)} \\
&\quad \leq \frac{\partial F_R}{1+\alpha}  \norm{ M_k \dot f (\theta_k, X_{k+1} ) }_2^{1+\alpha}\1{ \mathcal{B}_k(R)},
\end{align}
where $\partial F_R = \sup_{ \theta \in \overline{B(0,R)} } \inlinenorm{ \dot{F}(\theta)}_2 (1 + \alpha) < \infty$ given that $\inlinenorm{\dot{F}(\theta)}_2$ is a continuous function of $\theta$.

Applying this inequality to \cref{ineq-30sx4}, we conclude
\begin{equation}
\begin{aligned}
&[F(\theta_{k+1}) - \flb] \1{ \mathcal{B}_{k+1}(R) } \\
& \leq \left( [F(\theta_k) - \flb] - \dot{F}(\theta_k)'M_k \dot f(\theta_k, X_{k+1}) + \frac{L_{R+1} + \partial F_R}{1+\alpha} \norm{ M_k \dot f(\theta_k, X_{k+1}) }_2^{1+\alpha} \right) \\
& \quad \times \1{ \mathcal{B}_{k}(R) }.
\end{aligned}
\end{equation}

By \cref{assumption-unbiased},
\begin{equation}
\begin{aligned}
&\cond{[F(\theta_{k+1}) - \flb] \1{ \mathcal{B}_{k+1}(R) }}{\mathcal{F}_k} \\
&\leq \left( [F(\theta_k) - \flb] - \dot{F}(\theta_k)'M_k \dot{F}(\theta_k) + \frac{L_{R+1} + \partial F_R}{1+\alpha}\cond{ \norm{ M_k \dot{f}(\theta_k,X_{k+1}) }_2^{1+\alpha} }{\mathcal{F}_k} \right)\\
&\quad \times \1{ \mathcal{B}_{k}(R) }.
\end{aligned}
\end{equation}

Using \cref{prop-psd,assumption-moment},
\begin{equation}
\begin{aligned}
&\cond{[F(\theta_{k+1}) - \flb] \1{ \mathcal{B}_{k+1}(R) }}{\mathcal{F}_k} \\
&\leq \left( [F(\theta_k) - \flb] - \lambda_{\min}(M_k)\norm{\dot{F}(\theta_k)}_2^2 + \frac{L_{R+1} + \partial F_R}{1+\alpha} \lambda_{\max}(M_k)^{1+\alpha} G(\theta_k) \right) \1{ \mathcal{B}_{k}(R) }.
\end{aligned}
\end{equation}

By \cref{assumption-moment}, $G$ is upper semicontinuous and $\overline{B(0,R)}$ is compact, which implies that $G_R$ is well defined and finite.
The result follows.
\end{proof}

\subsection{Objective Function Analysis}

\begin{corollary} \label{corollary-convergence-objective-constrained}
Let $\lbrace \theta_k \rbrace$ be defined as in \cref{eqn-sgd-update} satisfying \cref{prop-psd,prop-max-eig}. 
Suppose \cref{assumption-flb,assumption-local-holder,assumption-unbiased,assumption-moment} hold.
Then, there exists a finite random variable $F_{\lim}$ such that on the event $\lbrace \sup_{k} \inlinenorm{\theta_k}_2 < \infty \rbrace$, $\lim_{k \to \infty} F(\theta_k) = F_{\lim}$ with probability one.
\end{corollary}
\begin{proof}
By \cref{lemma-recursion-constrained}, for every $R \geq 0$,
\begin{equation}
\begin{aligned}
&\cond{ [F(\theta_{k+1}) - \flb] \1{ \mathcal{B}_{k+1}(R) } }{\mathcal{F}_k} \\
&\quad \leq [F(\theta_k) - \flb] \1{ \mathcal{B}_{k}(R) }
+ \frac{(L_{R+1} + \partial F_R) G_R}{1+\alpha} \lambda_{\max}(M_k)^{1+\alpha}.
\end{aligned}
\end{equation}
By \citet[Exercise II.4]{neveu1975} (cf. \citet*{robbins1971}) and \cref{prop-max-eig}, $\lim_{k \to \infty} [ F(\theta_k) - \flb] \1{ \mathcal{B}_{k}(R)}$ converges to a finite random variable with probability one. Since $R \geq 0$ is arbitrary, we conclude that there exists a finite random variable $F_{\lim}$ such that the set $ \lbrace \sup_{k} \norm{\theta_k}_2 \leq R \rbrace$ is a subset of $\lbrace \lim_{k} F(\theta_k) = F_{\lim} \rbrace$ up to a measure zero set. Since the countable union of measure zero sets has measure zero,
\begin{equation}
\left\lbrace \sup_k \norm{\theta_k}_2 < \infty \right\rbrace = \bigcup_{ R \in \mathbb{N} } \left\lbrace \sup_k \norm{ \theta_k}_2 \leq R \right\rbrace \subset \left\lbrace \lim_{k \to \infty} F(\theta_k) = F_{\lim} \right\rbrace,
\end{equation}
up to a measure zero set. The result follows.
\end{proof}

\subsection{Gradient Function Analysis}

We now prove that the gradient norm evaluated at SGD's iterates must, repeatedly, get arbitrarily close to zero.
\begin{lemma} \label{lemma-liminf-grad-constrained}
Let $\lbrace \theta_k \rbrace$ be defined as in \cref{eqn-sgd-update} satisfying \cref{prop-psd,prop-max-eig,prop-min-eig}. 
Suppose \cref{assumption-flb,assumption-local-holder,assumption-unbiased,assumption-moment} hold. Then, $\forall R \geq 0$ and for all $\delta > 0$,
\begin{equation}
\Prb{ \norm{ \dot{F}(\theta_k)}_2^2 \1{\mathcal{B}_k(R)} \leq \delta, ~i.o. } = 1.
\end{equation}
\end{lemma}
\begin{proof}
By \cref{lemma-recursion-constrained},
\begin{equation}
\begin{aligned}
&\lambda_{\min}(M_k) \E{\norm{ \dot{F}(\theta_k)}_2^2 \1{\mathcal{B}_k(R)}} \leq \E{[F(\theta_k) - \flb] \1{ \mathcal{B}_{k}(R) }} \\
&\quad - \E{ [F(\theta_{k+1}) - \flb] \1{ \mathcal{B}_{k+1}(R) } } + \frac{(L_{R+1} + \partial F_R) G_R}{1+\alpha} \lambda_{\max}(M_k)^{1+\alpha}.
\end{aligned}
\end{equation}
Taking the sum of this equation for all $k$ from $0$ to $j \in \mathbb{N}$, we have
\begin{equation}
\begin{aligned}
& \sum_{k=0}^j \lambda_{\min}(M_k) \E{\norm{ \dot{F}(\theta_k)}_2^2 \1{\mathcal{B}_k(R)}} \leq [F(\theta_0) - \flb] \1{ \mathcal{B}_{0}(R) } \\
&\quad - \E{ [F(\theta_{j+1}) - \flb] \1{ \mathcal{B}_{j+1}(R) } } + \frac{(L_{R+1} + \partial F_R) G_R}{1+\alpha} \sum_{k=0}^j \lambda_{\max}(M_k)^{1+\alpha}.
\end{aligned}
\end{equation}
By \cref{assumption-flb,prop-max-eig}, the right hand side is bounded by
\begin{equation}
[F(\theta_0) - \flb] \1{ \mathcal{B}_{0}(R) } +  \frac{(L_{R+1} + \partial F_R) G_R}{1+\alpha}S,
\end{equation}
which is finite. Therefore, $\sum_{k=0}^\infty \lambda_{\min}(M_k) \inlineE{ \inlinenorm{ \dot{F}(\theta_k)}_2^2 \1{\mathcal{B}_k(R)}}$ is finite. Furthermore, by \cref{prop-min-eig}, $\liminf_k \inlineE{ \inlinenorm{ \dot{F}(\theta_k)}_2^2 \1{\mathcal{B}_k(R)}} = 0$.

Now, for any $\delta > 0$, Markov's inequality implies that for all $j +1 \in \mathbb{N}$, and for all $k \geq j$
\begin{align}
\Prb{ \bigcap_{k=j}^\infty \left\lbrace \norm{ \dot{F}(\theta_k)}_2^2 \1{\mathcal{B}_k(R)} > \delta \right\rbrace }
& \leq \Prb{\norm{ \dot{F}(\theta_k)}_2^2 \1{\mathcal{B}_k(R)} > \delta} \\
& \leq \frac{1}{\delta}\E{ \norm{ \dot{F}(\theta_k)}_2^2 \1{\mathcal{B}_k(R)} }.
\end{align}
Since the last inequality holds for every $k \geq j$, then, in particular, for all $j+1 \in \mathbb{N}$,
\begin{equation}
\Prb{ \bigcap_{k=j}^\infty \left\lbrace \norm{ \dot{F}(\theta_k)}_2^2 \1{\mathcal{B}_k(R)} > \delta \right\rbrace } 
\leq \frac{1}{\delta} \min_{j \leq k } \E{ \norm{ \dot{F}(\theta_k)}_2^2 \1{\mathcal{B}_k(R)} },
\end{equation}
where the right hand side is zero because $\liminf_k \inlineE{ \inlinenorm{ \dot{F}(\theta_k)}_2^2 \1{\mathcal{B}_k(R)}} = 0$.

As the countable union of measure zero sets has measure zero, we conclude that for all $\delta > 0$,
\begin{equation}
\Prb{ \norm{ \dot{F}(\theta_k)}_2^2 \1{\mathcal{B}_k(R)} \leq \delta, ~i.o. } = 1.
\end{equation}
\end{proof}

Unfortunately, \cref{lemma-liminf-grad-constrained} does not guarantee that the gradient norm will be captured within a region of zero. In order to prove this, we first show that it is not possible (i.e., a zero probability event) for the limit supremum and limit infimum of the gradients to be distinct.
\begin{lemma} \label{lemma-limsup-grad-constrained}
Let $\lbrace \theta_k \rbrace$ be defined as in \cref{eqn-sgd-update} satisfying \cref{prop-psd,prop-max-eig}. 
Suppose \cref{assumption-flb,assumption-local-holder,assumption-unbiased,assumption-moment} hold. 
Then, $\forall R \geq 0$ and for all $\delta > 0$,
\begin{equation}
\Prb{  \norm{ \dot{F}(\theta_{k+1})}_2 \1{ \mathcal{B}_{k+1}(R) } > \delta, \norm{ \dot{F}(\theta_k)}_2 \1{ \mathcal{B}_k(R) } \leq \delta , ~i.o.} = 0.
\end{equation}
\end{lemma}
\begin{proof}
Let $\gamma > 0$. Let $L_R$ be as in \cref{def-local-holder-constant}, and $G_R$ be as in \cref{lemma-recursion-constrained}. Then, for $\delta > 0$,
\begin{align}
&\Prb{ \norm{ \dot{F}(\theta_{k+1})}_2 \1{ \mathcal{B}_{k+1}(R) } \1{ \norm{ \dot{F}(\theta_k)}_2 \1{ \mathcal{B}_k(R) } \leq \delta } > \delta + L_R \gamma^\alpha } \\
&= \mathbb{P} \bigg{[} \left( \norm{ \dot{F}(\theta_{k+1})}_2 - \norm{\dot{F}(\theta_k)}_2 + \norm{\dot F(\theta_k)}_2 \right) \1{ \mathcal{B}_{k+1}(R) }  \\
&\quad \times \1{ \norm{ \dot{F}(\theta_k)}_2 \1{ \mathcal{B}_k(R) } \leq \delta } > \delta + L_R \gamma^\alpha  \bigg{]}.
\end{align}
Using the reverse triangle inequality, $\inlinenorm{ \dot F(\theta_{k+1}) }_2 - \inlinenorm{ \dot F(\theta_k)}_2 \leq \inlinenorm{ \dot F(\theta_{k+1}) - \dot F(\theta_k)}_2$. Now, making use of the restriction to $\mathcal{B}_{k+1}(R)$, $\inlinenorm{ \dot F(\theta_{k+1}) - \dot F(\theta_k)}_2 \leq L_R \inlinenorm{ \theta_{k+1} - \theta_k }_2^\alpha$. Moreover, on $\inlinenorm{ \dot F (\theta_k) }_2 \leq \delta$, $\inlinenorm{ \dot F(\theta_k)}_2 \1{ \mathcal{B}_{k+1}(R) } \leq \delta$. Putting these two observations together,
\begin{align}
&\mathbb{P} \bigg{[} \left( \norm{ \dot{F}(\theta_{k+1})}_2 - \norm{\dot{F}(\theta_k)}_2 + \norm{\dot F(\theta_k)}_2 \right) \1{ \mathcal{B}_{k+1}(R) }  \\
&\quad \times \1{ \norm{ \dot{F}(\theta_k)}_2 \1{ \mathcal{B}_k(R) } \leq \delta } > \delta + L_R \gamma^\alpha  \bigg{]} \\
&\leq \mathbb{P} \bigg{[} L_R \norm{ \theta_{k+1} - \theta_k}_2^\alpha \1{ \mathcal B_{k+1}(R) }\1{ \norm{ \dot{F}(\theta_k)}_2 \1{ \mathcal{B}_k(R) } \leq \delta } > L_R \gamma^\alpha\bigg{]} \\
&= \mathbb{P} \bigg{[} \norm{ M_k \dot f(\theta_k, X_{k+1} )}_2 \1{ \mathcal B_{k+1}(R) }\1{ \norm{ \dot{F}(\theta_k)}_2 \1{ \mathcal{B}_k(R) } \leq \delta } > \gamma \bigg{]}.
\end{align}
Now, using $\1{ \mathcal B_{k+1}(R) }\1{ \norm{ \dot{F}(\theta_k)}_2 \1{ \mathcal{B}_k(R) }} \leq \1{ \mathcal{B}_k(R)}$,
\begin{align}
&\mathbb{P} \bigg{[} \norm{ M_k \dot f(\theta_k, X_{k+1} )}_2 \1{ \mathcal B_{k+1}(R) }\1{ \norm{ \dot{F}(\theta_k)}_2 \1{ \mathcal{B}_k(R) } \leq \delta } > \gamma \bigg{]} \\
&\leq \Prb{ \norm{M_k \dot f(\theta_k, X_{k+1} ) }_2 \1{ \mathcal B_k(R) } > \gamma}\\
&\leq \Prb{ \norm{M_k \dot f(\theta_k, X_{k+1} ) }_2^{1+\alpha} \1{ \mathcal B_k(R)} > \gamma^{1+\alpha}} \\
&\leq \frac{1}{\gamma^{1+\alpha}} \norm{M_k}_2^{1+\alpha} \E{ \cond{ \norm{ \dot f(\theta_k, X_{k+1}) }_2^{1+\alpha}}{\mathcal{F}_k}\1{ \mathcal B_k(R)} },
\end{align}
where the last inequality is a consequence of Markov's inequality, $\inlinenorm{ M_k \dot f(\theta_k, X_{k+1})}_2\leq\inlinenorm{M_k}_2 \inlinenorm{ \dot f(\theta_k,X_{k+1})}_2$, and $\1{\mathcal{B}_k(R)}$ being measurable with respect to $\mathcal{F}_k$.

By \cref{assumption-moment}, $\cond{ \inlinenorm{ \dot f(\theta_k, X_{k+1}) }_2^{1+\alpha}}{\mathcal{F}_k} \leq G(\theta_k)$. Moreover, on $\mathcal B_k(R)$, $G(\theta_k) \leq G_R$. Using this in the expectation, we conclude
\begin{align}
\Prb{ \norm{ \dot{F}(\theta_{k+1})}_2 \1{ \mathcal{B}_{k+1}(R) } \1{ \norm{ \dot{F}(\theta_k)}_2 \1{ \mathcal{B}_k(R) } \leq \delta } > \delta + L_R \gamma^\alpha }  \leq \frac{1}{\gamma^{1+\alpha}} \norm{M_k}_2^{1+\alpha}  G_R.
\end{align}
By \cref{prop-max-eig}, the sum of the last expression over all $k+1 \in \mathbb{N}$ is finite. By the Borel-Cantelli lemma, for all $R \geq 0$, $\delta > 0$ and $\gamma > 0$,
\begin{equation}
\Prb{  \norm{ \dot{F}(\theta_{k+1})}_2 \1{ \mathcal{B}_{k+1}(R) } > \delta + L_R \gamma^\alpha, \norm{ \dot{F}(\theta_k)}_2 \1{ \mathcal{B}_k(R) } \leq \delta , ~i.o.} = 0.
\end{equation}
Since this holds for any $\gamma > 0$, it will hold for every value in a sequence $\gamma_n \downarrow 0$. Since the countable union of measure zero events has measure zero, for any $R \geq 0$ and $\delta > 0$, 
\begin{equation}  \label{eqn-omega-delta-gap}
\Prb{ \left\lbrace  \norm{ \dot{F}(\theta_{k+1})}_2 \1{ \mathcal{B}_{k+1}(R) } > \delta, \norm{ \dot{F}(\theta_k)}_2 \1{ \mathcal{B}_k(R) } \leq \delta , ~i.o. \right\rbrace \cap \Omega_\delta^c}= 0,
\end{equation}
where $\Omega_\delta = \lbrace \limsup_{k} \inlinenorm{ \dot F(\theta_{k+1})}_2 \1{ \mathcal{B}_{k+1}(R)} = \delta \rbrace$.

We now show that $\Omega_\delta$ is a probability zero event. Notice, by \cref{lemma-liminf-grad-constrained} and the definition of $\Omega_{\delta}$,
\begin{equation}
\Omega_\delta \subset \left\lbrace \norm{ \dot{F}(\theta_{k+1})}_2 \1{ \mathcal{B}_{k+1}(R) } > \delta/2, \norm{ \dot{F}(\theta_k)}_2 \1{ \mathcal{B}_k(R) } \leq \delta/2 , ~i.o. \right\rbrace \cap \Omega_{\delta/2}^c,
\end{equation}
up to a set of measure zero. By applying \cref{eqn-omega-delta-gap} with $\delta/2$, $\inlinePrb{\Omega_\delta} = 0$. The conclusion of the result follows.
\end{proof}

We now put together \cref{lemma-liminf-grad-constrained,lemma-limsup-grad-constrained} to show that, on the event $\lbrace \sup_k \inlinenorm{\theta_k}_2 < \infty \rbrace$, $\inlinenorm{ \dot{F}(\theta_k) }_2$ converges to $0$ with probability one.

\begin{corollary} \label{corollary-convergence-grad-constrained}
Let $\lbrace \theta_k \rbrace$ be defined as in \cref{eqn-sgd-update} satisfying \cref{prop-psd,prop-max-eig,prop-min-eig}. 
Suppose \cref{assumption-flb,assumption-local-holder,assumption-unbiased,assumption-moment} hold. 
Then, on the event $\lbrace \sup_k \inlinenorm{\theta_k}_2 < \infty \rbrace$, $\lim_{k \to \infty} \inlinenorm{ \dot{F}(\theta_k)}_2 = 0$.
\end{corollary}
\begin{proof}
For any $R \geq 0$ and $\delta > 0$, \cref{lemma-liminf-grad-constrained} implies
\begin{equation}
\begin{aligned}
&\Prb{ \norm{\dot{F}(\theta_{k+1})}_2 \1{ \mathcal{B}_{k+1}(R) } > \delta, ~i.o. } \\
&\quad = \Prb{ \left\lbrace \norm{\dot{F}(\theta_{k+1})}_2 \1{ \mathcal{B}_{k+1}(R) } > \delta,~i.o. \right\rbrace \cap \left\lbrace \norm{\dot{F}(\theta_{k})}_2 \1{ \mathcal{B}_{k}(R) } \leq \delta, ~i.o. \right\rbrace }.
\end{aligned}
\end{equation}
We see that this latter event is exactly, 
\begin{equation}
\Prb{  \norm{ \dot{F}(\theta_{k+1})}_2 \1{ \mathcal{B}_{k+1}(R) } > \delta, \norm{ \dot{F}(\theta_k)}_2 \1{ \mathcal{B}_k(R) } \leq \delta , ~i.o.},
\end{equation}
which, by \cref{lemma-limsup-grad-constrained}, is zero with probability one. Therefore, $\inlinePrb{ \inlinenorm{\dot F(\theta_{k+1})}_2 \1{ \mathcal{B}_{k+1}(R) } > \delta, ~i.o. }$ is zero. Letting $\delta_n \downarrow 0$ and noting that the countable union of measure zero sets has measure zero, we conclude $\inlinePrb{ \inlinenorm{\dot F(\theta_{k+1})}_2 \1{ \mathcal{B}_{k+1}(R) } > 0, ~i.o. } = 0$.

Therefore, for all $R \geq 0$, $\lbrace \sup_k \norm{\theta_k}_2 \leq R \rbrace \subset \lbrace \lim_{k\to\infty} \inlinenorm{\dot{F}(\theta_k)}_2 = 0 \rbrace$ up to a measure zero set. Since $\lbrace \sup_k \inlinenorm{\theta_k}_2 < \infty \rbrace = \cup_{R \in \mathbb{N} } \lbrace \sup_k \inlinenorm{\theta_k}_2 \leq R \rbrace$, the result follows.
\end{proof}

\subsection{Capture Theorem}

The final step in our proof is to study the event $\lbrace \sup_k \inlinenorm{\theta_k} < \infty \rbrace$. 

\begin{theorem}[\cref{theorem-capture}]
Let $\lbrace \theta_k \rbrace$ be defined as in \cref{eqn-sgd-update}, and let $\lbrace M_k \rbrace$ satisfy \cref{prop-psd,prop-max-eig}. If \cref{assumption-moment} holds, then either $\lbrace \lim_{k \to \infty} \theta_k ~\mathrm{exists} \rbrace$ or $\lbrace \liminf_{k\to\infty} \inlinenorm{ \theta_k}_2 = \infty \rbrace$ must occur.
\end{theorem}
\begin{proof}
Let $\bar \theta \in \mathbb{R}^p$. Fix $ R \geq 0 $ and let $\gamma > 0$. Then,
\begin{align}
&\Prb{ \norm{ \theta_{k+1} - \bar \theta}_2 \geq R + \gamma, \norm{ \theta_k - \bar \theta}_2 \leq R } \nonumber \\
&\quad = \Prb{ \norm{ \theta_{k+1} - \bar \theta}_2 \1{ \norm{ \theta_k - \bar \theta}_2 \leq R } \geq R + \gamma } \\
&\quad = \Prb{ \left( \norm{ \theta_{k+1} - \bar \theta}_2 - \norm{ \theta_k - \bar \theta}_2  + \norm{ \theta_k - \bar \theta}_2  \right) \1{ \norm{ \theta_k - \bar \theta}_2 \leq R } \geq R + \gamma }.
\end{align}
Now, $\inlinenorm{ \theta_k - \bar \theta}_2 \1{ \inlinenorm{ \theta_k - \bar \theta}_2 \leq R} \leq R$. Therefore,
\begin{align} 
&\Prb{ \left( \norm{ \theta_{k+1} - \bar \theta}_2 - \norm{ \theta_k - \bar \theta}_2  + \norm{ \theta_k - \bar \theta}_2  \right) \1{ \norm{ \theta_k - \bar \theta}_2 \leq R } \geq R + \gamma } \\
&\quad \leq \Prb{ \left( \norm{ \theta_{k+1} - \bar \theta}_2 - \norm{ \theta_k - \bar \theta}_2  \right) \1{ \norm{ \theta_k - \bar \theta}_2 \leq R } + R \geq R + \gamma } \\
&\quad \leq \Prb{ \norm{ \theta_{k+1} -  \theta_k }_2  \1{ \norm{ \theta_k - \bar \theta}_2 \leq R } \geq \gamma }, 
\end{align}
where the last line follows by applying the reverse triangle inequality. By using \cref{eqn-sgd-update} and Markov's inequality,
\begin{align}
&\Prb{ \norm{ \theta_{k+1} -  \theta_k }_2  \1{ \norm{ \theta_k - \bar \theta}_2 \leq R } \geq \gamma } \\
&\quad \leq \Prb{ \norm{ M_k \dot f(\theta_k, X_{k+1})}_2 \1{ \norm{ \theta_k - \bar \theta}_2 \leq R } \geq \gamma } \\
&\quad \leq \frac{1}{\gamma^{1+\alpha}} \norm{M_k}_2^{1+\alpha} \E{ \cond{\norm{ \dot f(\theta_k, X_{k+1} )}_2^{1+\alpha} }{ \mathcal{F}_k} \1{ \norm{ \theta_k - \bar \theta}_2 \leq R }}.
\end{align}
By applying \cref{assumption-moment}, $\cond{ \inlinenorm{ \dot f(\theta_k, X_{k+1} )}_2^{1+\alpha}}{\mathcal{F}_k} \leq G(\theta_k)$. Moreover, on $\inlinenorm{ \theta_k - \bar \theta}_2 \leq R$, $G(\theta_k) \leq \sup_{ \theta: \inlinenorm{\theta}_2 \leq R + \inlinenorm{\bar\theta}_2}  G(\theta) =: G_{R + \inlinenorm{\bar \theta}_2} < \infty$ since $G$ is upper semi-continuous. Combining these steps,
\begin{align}
\Prb{ \norm{ \theta_{k+1} - \bar \theta}_2 \geq R + \gamma, \norm{ \theta_k - \bar \theta}_2 \leq R } \leq \frac{1}{\gamma^{1+\alpha}} \norm{M_k}_2^{1+\alpha} G_{R + \inlinenorm{\bar \theta}_2},
\end{align}
By \cref{prop-max-eig}, we see that the sum of the probabilities is finite. Together with the Borel-Cantelli lemma, $\forall R \geq 0$, $\forall \gamma > 0$, and for all $\bar \theta \in \mathbb{R}^p$,
\begin{equation}
\Prb{ \inlinenorm{ \theta_{k+1} - \bar \theta}_2 \geq R + \gamma, ~ \inlinenorm{ \theta_k - \bar \theta}_2 \leq R,~i.o.} = 0.
\end{equation}
Since $\gamma > 0$ is arbitrary, we can show that this statement holds for a countable sequence of $\gamma_n \downarrow 0$. Therefore, $\forall R \geq 0$ and all $\bar \theta \in \mathbb{R}^p$,
\begin{equation}
\Prb{ \limsup_{k} \norm{ \theta_k - \bar \theta}_2 > R, \liminf_k \norm{\theta_k - \bar \theta_2}_2 \leq R} = 0.
\end{equation}

Since $R$ is arbitrary, we conclude that for any ordering of positive rational numbers, $\lbrace R_n \rbrace$, $\inlinePrb{\limsup_k \inlinenorm{ \theta_{k+1} - \bar \theta}_2 > R_n,  \liminf_k \inlinenorm{ \theta_k - \bar \theta}_2 \leq R_n} = 0$ for every $n$. Again, the countable union of measure zero sets is measure zero. Hence, we conclude that $\inlinePrb{ \limsup_{k} \inlinenorm{\theta_{k} - \bar \theta}_2 > \liminf_k \inlinenorm{ \theta_k - \bar \theta}_2} = 0$. Thus, either $\lim_k \inlinenorm{\theta_k - \bar \theta}_2$ exists and is either infinite or finite. 

Moreover, on the event that the limit is finite, since $\bar \theta$ is arbitrary, we can choose $p+1$ distinct values of $\bar{\theta}$ which do not belong to a hyperplane of dimension smaller than $p$, and, by triangulation, the $\lim_k \theta_k$ converges to a fixed point (up to a set of measure zero).
\end{proof}

%% file: section/stability-analysis.tex
We begin with a recursive relationship on the events $\lbrace \tau_j > k \rbrace$. We use this result to prove that the objective function converges to a finite limit on these events. Then, we use this result to conclude that the gradient function visits to a region of zero on the same event. Finally, we study this event to establish that the two statements above hold with probability one.

\subsection{A Recursive Relationship}

\begin{lemma}[\cref{lemma-optimality-recursion}]
Let $\lbrace M_k \rbrace$ satisfy \cref{prop-psd}.
Suppose \cref{assumption-flb,assumption-local-holder,assumption-unbiased,assumption-moment} hold.
Let  $\lbrace \theta_k \rbrace$ satisfy \cref{eqn-sgd-update}.
Then, for any $j+1 \in \mathbb{N}$ and $k > j$, 
\begin{equation}
\begin{aligned}
&\cond{ \left( F(\theta_{k+1} ) - \flb \right) \1{ \tau_j > k } }{ \mathcal{F}_k} 
\leq \left( F(\theta_{k}) - \flb - \dot F(\theta_k)' M_k \dot F(\theta_k) \right) \1{ \tau_j > k-1} \\
&\quad + \frac{\lambda_{\max}(M_k)^{1+\alpha}}{1+\alpha}\left[\mathcal{L}_{\epsilon}(\theta_k) G(\theta_k) + \alpha \left[\frac{ \norm{ \dot F(\theta_k) }_2^{1+\alpha} }{ \mathcal{L}_{\epsilon}(\theta_k) }  \right]^{1/\alpha}  \right] \1{\tau_j > k - 1}.
\end{aligned}
\end{equation}
\end{lemma}
\begin{proof}
By the construction of $\tau_j$, when $\tau_j > k$, then
\begin{equation}
F(\theta_{k+1} ) - \flb \leq F(\theta_k) - \flb + \dot F(\theta_k)'(\theta_{k+1} - \theta_k) + \frac{\mathcal{L}_{\epsilon}(\theta_k)}{1+\alpha} \norm{ \theta_{k+1} - \theta_k }_2^{1+\alpha}.
\end{equation}
Using this relationship and using $\1{ \tau_j > k } = \1{ \tau_j > k - 1} - \1{ \tau_j = k }$,
\begin{equation} \label{eqn-xe49d}
\begin{aligned}
&\cond{ \left\lbrace F(\theta_{k+1} ) - \flb \right\rbrace \1{ \tau_j > k } }{\mathcal{F}_k}  \\
&\leq \cond{ \left\lbrace F(\theta_k) - \flb + \dot F(\theta_k)'(\theta_{k+1} - \theta_k) + \frac{\mathcal{L}_{\epsilon}(\theta_k)}{1+\alpha} \norm{ \theta_{k+1} - \theta_k }_2^{1+\alpha} \right\rbrace \1{ \tau_j > k - 1} }{\mathcal{F}_k} \\
& - \cond{ \left\lbrace F(\theta_k) - \flb + \dot F(\theta_k)'(\theta_{k+1} - \theta_k) + \frac{\mathcal{L}_{\epsilon}(\theta_k)}{1+\alpha} \norm{ \theta_{k+1} - \theta_k }_2^{1+\alpha} \right\rbrace \1{ \tau_j = k }}{\mathcal{F}_k}
\end{aligned}
\end{equation}
For the first term on the right hand side, we can apply \cref{assumption-unbiased,assumption-moment,prop-psd,eqn-sgd-update} to calculate
\begin{equation} \label{eqn-xe49d-1}
\begin{aligned}
&\cond{ \left\lbrace F(\theta_k) - \flb + \dot F(\theta_k)'(\theta_{k+1} - \theta_k) + \frac{\mathcal{L}_{\epsilon}(\theta_k)}{1+\alpha} \norm{ \theta_{k+1} - \theta_k }_2^{1+\alpha} \right\rbrace \1{ \tau_j > k - 1} }{\mathcal{F}_k} \\
&\leq \left\lbrace F(\theta_k) - \flb - \dot F(\theta_k)'M_k \dot F(\theta_k) + \frac{\lambda_{\max}(M_k)^{1+\alpha}}{1+\alpha} \mathcal{L}_{\epsilon}(\theta_k) G(\theta_k) \right\rbrace \1{ \tau_j > k - 1}.
\end{aligned}
\end{equation}
For the second term on the right hand side of \cref{eqn-xe49d}, we require two facts.
The first fact is $\1{ \tau_j = k } \leq \1{ \tau_j > k - 1}$ which implies $\1{ \tau_j = k } = \1{ \tau_j = k } \1{ \tau_j > k - 1}$. 
For the second fact, the Cauchy-Schwarz inequality and \cref{lemma-taylor-lower-bound} imply
\begin{align}
&- F(\theta_k)'M_k \dot f(\theta_k,X_{k+1}) + \frac{\mathcal{L}_{\epsilon}(\theta_k)}{1+\alpha} \norm{ M_k f(\theta_k, X_{k+1}) }_2^{1+\alpha}  \nonumber\\
& \geq - \norm{ F(\theta_k)}_2 \norm{ M_k \dot f(\theta_k, X_{k+1} )}_2 + \frac{\mathcal{L}_{\epsilon}(\theta_k)}{1+\alpha} \norm{ M_k f(\theta_k, X_{k+1}) }_2^{1+\alpha} \\
& \geq - \frac{\alpha}{1+\alpha} \left[ \frac{ \norm{ \dot F(\theta_k) }_2^{1+\alpha} }{ \mathcal{L}_{\epsilon}(\theta_k) }  \right]^{1/\alpha}.
\end{align}
Hence, using \cref{eqn-sgd-update},
\begin{equation}
\begin{aligned}
&- \left[F(\theta_k) - \flb \right] - \dot{F}(\theta_k)'(\theta_{k+1} - \theta_k) - \frac{\mathcal{L}_{\epsilon}(\theta_k)}{1+\alpha} \norm{ \theta_{k+1} - \theta_k}_2^{1+\alpha} \\
& \leq - \left[F(\theta_k) - \flb \right] + \frac{\alpha}{1+\alpha} \left[ \frac{ \norm{ \dot F(\theta_k) }_2^{1+\alpha} }{ \mathcal{L}_{\epsilon}(\theta_k) }  \right]^{1/\alpha}.
\end{aligned}
\end{equation}
Putting together these two preceding facts together,
\begin{align}
&- \cond{ \left\lbrace F(\theta_k) - \flb + \dot F(\theta_k)'(\theta_{k+1} - \theta_k) + \frac{\mathcal{L}_{\epsilon}(\theta_k)}{1+\alpha} \norm{ \theta_{k+1} - \theta_k }_2^{1+\alpha} \right\rbrace \1{ \tau_j = k }}{\mathcal{F}_k} \nonumber \\
&\leq \left\lbrace - \left[ F(\theta_k) - \flb \right] +  \frac{\alpha}{1+\alpha} \left[ \frac{ \norm{ \dot F(\theta_k) }_2^{1+\alpha} }{ \mathcal{L}_{\epsilon}(\theta_k) }  \right]^{1/\alpha} \right\rbrace \condPrb{ \tau_{j} = k }{\mathcal{F}_k} \1{ \tau_j > k - 1} \\
&\leq \frac{\alpha \lambda_{\max}(M_k)^{1+\alpha}}{1+\alpha} \left[ \frac{ \norm{ \dot F(\theta_k) }_2^{1+\alpha} }{ \mathcal{L}_{\epsilon}(\theta_k) }  \right]^{1/\alpha}  \1{ \tau_j > k - 1}, \label{eqn-xed49d-2}
\end{align}
where we bound $\inlinecondPrb{ \tau_j = k}{\mathcal{F}_k}$ using  \cref{theorem-probability-stop-times}.
By applying the bounds on the first term, \cref{eqn-xe49d-1}, and second term, \cref{eqn-xed49d-2}, to \cref{eqn-xe49d}, the result follows.
\end{proof}

By applying \cref{assumption-stability} to \cref{lemma-optimality-recursion}, we have the following simplified form.
\begin{lemma}[\cref{lemma-optimality-recursion-simple}]
If \cref{assumption-flb,assumption-local-holder,assumption-unbiased,assumption-moment,assumption-stability}, and \cref{prop-psd,prop-cond-num} hold, and $\lbrace \theta_k \rbrace$ satisfy \cref{eqn-sgd-update}, then there exists a $K \in \mathbb{N}$ such that for any $j+1 \in \mathbb{N}$ and any $k \geq \min\lbrace K, j+1 \rbrace$, 
\begin{equation}
\begin{aligned}
&\cond{ (F(\theta_{k+1}) - \flb) \1{ \tau_j > k } }{\mathcal{F}_k}  \\
&\leq \left( 1 + \lambda_{\max}(M_k)^{1+\alpha} \frac{C_2}{1+\alpha} \right) (F(\theta_k) - \flb) \1{ \tau_j > k - 1} \\
& - \frac{1}{2} \lambda_{\min}(M_k) \norm{ \dot F (\theta_k) }_2^2 \1{ \tau_j > k - 1} + \lambda_{\max}(M_k)^{1+\alpha}\frac{C_1}{1+\alpha}.
\end{aligned}
\end{equation}
\end{lemma}
\begin{proof}
The result follows by first using \cref{assumption-stability} in \cref{lemma-optimality-recursion}. Then, collecting similar terms, we apply \cref{lemma-lr-eig-bound} to find $K$.
\end{proof}

\subsection{Objective Function Analysis}
With this recursive formula, we now have the first result.
\begin{corollary} \label{corollary-stability-union-stop-times}
If \cref{assumption-flb,assumption-local-holder,assumption-unbiased,assumption-moment,assumption-stability,prop-psd,prop-max-eig,prop-cond-num} hold, and $\lbrace \theta_k \rbrace$ satisfy \cref{eqn-sgd-update}, then $\lim_{k \to \infty} F(\theta_{k})$ exists and is finite on $\cup_{j=0}^\infty \lbrace \tau_j = \infty \rbrace$.
\end{corollary}
\begin{proof}
By \cref{lemma-optimality-recursion-simple} and \citet[Exercise II.4]{robbins1971,neveu1975}, the limit as $k$ goes to infinity of $(F(\theta_{k}) - \flb)\1{ \tau_j > k-1}$ exists with probability one and is integrable. Therefore, on the event $\lbrace \tau_j = \infty \rbrace$, the limit of $F(\theta_k) - \flb$ exists and is integrable. As a result, the limit of $F(\theta_k) - \flb$ exists and is finite on $\cup_{j=0}^\infty \lbrace \tau_j = \infty \rbrace$.
\end{proof}
Additionally, we can state the following useful result.

\begin{lemma} \label{lemma-expected-optimality-gap-bound}
If \cref{assumption-flb,assumption-local-holder,assumption-unbiased,assumption-moment,assumption-stability}, and \cref{prop-psd,prop-max-eig,prop-cond-num} hold, and $\lbrace \theta_k \rbrace$ satisfy \cref{eqn-sgd-update}, then $ \exists K \in \mathbb{N}$ such that for any $j>K$, $\exists N_j > 0$ for which
\begin{equation}
\sup_{k > j} \E{ ( F(\theta_k) - \flb ) \1{ \tau_j > k -1 } } \leq N_j.
\end{equation}
\end{lemma}
\begin{proof}
In \cref{lemma-optimality-recursion-simple}, we upper bound the right hand side by removing the negative term, and, by \cref{prop-max-eig}, we add $C_1(1+\alpha)^{-1} \sum_{\ell=k+1}^\infty \lambda_{\max}(M_\ell)^{1+\alpha}$ to both side. Then, for all $k \geq j$,
\begin{equation}
\begin{aligned}
&\E{ ( F(\theta_{k+1}) - \flb ) \1{ \tau_j > k }} + \frac{C_1}{1+\alpha} \sum_{\ell = k+1}^\infty \lambda_{\max}(M_\ell)^{1+\alpha} \\
&\leq \left(1 + \lambda_{\max}(M_k)^{1+\alpha} \frac{C_2}{1+\alpha}\right) \E{ (F(\theta_k) - \flb) \1{ \tau_j > k - 1} } + \frac{C_1}{1+\alpha} \sum_{\ell = k }^\infty \lambda_{\max}(M_\ell)^{1+\alpha}.
\end{aligned}
\end{equation}

Using $1 + C_2 (1+\alpha)^{-1} \lambda_{\max}(M_k)^{1+\alpha} \leq \exp( C_2 (1+\alpha)^{-1} \lambda_{\max}(M_k)^{1+\alpha})$, it follows
\begin{equation}
\begin{aligned}
&\E{ ( F(\theta_{k+1}) - \flb ) \1{ \tau_j > k }} + \frac{C_1}{1+\alpha} \sum_{\ell = k+1}^\infty \lambda_{\max}(M_\ell)^{1+\alpha} \\
&\leq \exp\left( \frac{C_2}{1+\alpha} \lambda_{\max}(M_k)^{1+\alpha} \right) \left[ \E{ (F(\theta_k) - \flb) \1{ \tau_j > k - 1} } + \frac{C_1}{1+\alpha} \sum_{\ell = k }^\infty \lambda_{\max}(M_\ell)^{1+\alpha} \right].
\end{aligned}
\end{equation}
Hence,
\begin{equation}
\begin{aligned}
&\E{ ( F(\theta_{k+1}) - \flb ) \1{ \tau_j > k }} + \frac{C_1}{1+\alpha} \sum_{\ell = k+1}^\infty \lambda_{\max}(M_\ell)^{1+\alpha} \\
&\leq \exp\left( \frac{C_2}{1+\alpha} \sum_{\ell = j}^k \lambda_{\max}(M_\ell)^{1+\alpha} \right) \left[ \E{ (F(\theta_j) - \flb)} + \frac{C_1}{1+\alpha} \sum_{\ell = j }^\infty \lambda_{\max}(M_\ell)^{1+\alpha} \right],
\end{aligned}
\end{equation}
where we have used $\1{ \tau_j > j - 1} = 1$. By \cref{prop-max-eig}, the summation in the exponent is finite, which implies the result.
\end{proof}

\subsection{Gradient Function Analysis}

\begin{lemma} \label{lemma-liminf-gradient}
If \cref{assumption-flb,assumption-local-holder,assumption-unbiased,assumption-moment,assumption-stability}, and \cref{prop-psd,prop-max-eig,prop-min-eig,prop-cond-num} hold, and $\lbrace \theta_k \rbrace$ satisfy \cref{eqn-sgd-update}, then, for any $\delta > 0$,
\begin{equation}
\Prb{ \norm{ \dot F (\theta_k) }_2 \1{ \tau_j > k - 1} \leq \delta ~ i.o. } = 1.
\end{equation}
\end{lemma}
\begin{proof}
By \cref{lemma-optimality-recursion-simple},
\begin{equation}
\begin{aligned}
& \frac{1}{2}\lambda_{\min}(M_k)\E{ \norm{ \dot F (\theta_k) }_2^2 \1{ \tau_j > k - 1} } 
\leq \E{ ( F(\theta_k) - \flb) \1{ \tau_j > k - 1 }} \\
&\quad - \E{ ( F(\theta_{k+1}) - \flb) \1{ \tau_j > k }} + \frac{C_2}{1+\alpha} \lambda_{\max}(M_k)^{1+\alpha} \E{ ( F(\theta_k) - \flb) \1{ \tau_j > k - 1 }} \\
&\quad + \frac{C_1}{1+\alpha} \lambda_{\max}(M_k)^{1+\alpha}.
\end{aligned}
\end{equation}
By applying \cref{lemma-expected-optimality-gap-bound},
\begin{equation}
\begin{aligned}
& \frac{1}{2}\lambda_{\min}(M_k)\E{ \norm{ \dot F (\theta_k) }_2^2 \1{ \tau_j > k - 1} } 
\leq \E{ ( F(\theta_k) - \flb) \1{ \tau_j > k - 1 }} \\
&\quad - \E{ ( F(\theta_{k+1}) - \flb) \1{ \tau_j > k }} + \lambda_{\max}(M_k)^{1+\alpha} \left( \frac{C_2 N_j + C_1}{1+\alpha} \right).
\end{aligned}
\end{equation}
By summing and using \cref{assumption-flb},
\begin{equation}
\begin{aligned}
& \frac{1}{2} \sum_{k = j}^\infty \lambda_{\min}(M_k)  \E{ \norm{ \dot F(\theta_k)}_2^2 \1{ \tau_j > k -1}} \\
&\quad\leq \E{ F(\theta_j) - \flb } + \frac{C_2 N_j + C_1}{1+\alpha} \sum_{k=j}^\infty \lambda_{\max}(M_k)^{1+\alpha}.
\end{aligned}
\end{equation}
By \cref{prop-max-eig}, the right hand side is bounded. Now, by \cref{prop-min-eig},
\begin{equation}
\liminf_{k \to \infty} \E{ \norm{ \dot F(\theta_k)}_2^2 \1{ \tau_j > k -1 }} = 0.
\end{equation}
Using Markov's inequality, for any $\ell \in \mathbb{N}$ and any $\delta > 0$,
\begin{equation}
\Prb{ \bigcap_{k = \ell}^\infty \left\lbrace \norm{ \dot F(\theta_k) }_2 \1{ \tau_j > k -1} > \delta \right\rbrace } \leq \frac{1}{\delta^2} \min_{k \geq \ell} \E{ \norm{ \dot F(\theta_k)}_2^2 \1{ \tau_j > k - 1 }} = 0.
\end{equation}
As the countable union of sets of measure zero have measure zero, the result follows.
\end{proof}

\subsection{Stopping Time Analysis}

WE compute the probability of $\lbrace \tau_j = k \rbrace$.

\begin{theorem} \label{theorem-probability-stop-times}
Let $\lbrace \tau_j : j+1 \in \mathbb{N} \rbrace$ be defined as in \cref{eqn-tau-j}. If \cref{assumption-flb,assumption-local-holder,assumption-moment,prop-psd} hold, and $\lbrace \theta_k \rbrace$ satisfy \cref{eqn-sgd-update}, then, for any $j + 1 \in \mathbb{N}$ and any $k + 1 \in \mathbb{N}$,
\begin{equation}
\condPrb{ \tau_j = k }{\mathcal{F}_k} \leq \begin{cases}
0 & k \leq j, \\
\lambda_{\max}(M_k)^{1+\alpha} & k > j.
\end{cases}
\end{equation}
Moreover, if \cref{prop-max-eig} also holds, then $\Prb{ \cup_{j=0}^\infty \lbrace \tau_j = \infty \rbrace } = 1$.
\end{theorem}
\begin{proof}
The case of $k \leq j$ is trivial. So consider only $k > j$.
By the construction of $L(\cdot,\cdot)$ and $\mathcal{L}_{\epsilon}(\cdot)$, $\omega \in \lbrace L(\theta_k,\theta_{k+1}) > \mathcal{L}_{\epsilon}(\theta_k) \rbrace$ implies $\omega \in \lbrace \norm{\theta_{k+1} - \theta_k}_2 > ( G(\theta_k) \vee \epsilon )^{ \frac{1}{1+\alpha} }\rbrace$. Using \cref{eqn-sgd-update}, Markov's inequality, \cref{prop-psd}, we conclude
\begin{align}
\condPrb{ \tau_j = k }{\mathcal{F}_k} 
&\leq \condPrb{ \norm{ M_k \dot f(\theta_k, X_{k+1} ) }_2^{1+\alpha} > G(\theta_k) \vee \epsilon }{\mathcal{F}_k} \\
&\leq \frac{\lambda_{\max}(M_k)^{1+\alpha} \cond{\norm{ \dot f(\theta_k, X_{k+1} }_2^{1+\alpha}}{\mathcal{F}_k} }{G(\theta_k) \vee \epsilon}.
\end{align}
Applying \cref{assumption-moment} supplies the bound on $\condPrb{ \tau_j = k }{\mathcal{F}_k}$. For the second part, note
\begin{equation}
\Prb{ \tau_j = \infty } \geq  1 - \Prb{ \tau_j < \infty} \geq 1 - \sum_{k = j+1}^\infty \lambda_{\max}(M_k)^{1+\alpha}.
\end{equation}
Therefore,
\begin{equation}
\Prb{ \bigcup_{j=0}^\infty \lbrace \tau_j = \infty \rbrace } = \lim_{j \to \infty} \Prb{ \tau_j = \infty }. 
\end{equation}
Since $\lim_j \inlinePrb{ \tau_j = \infty} \geq 1 - \lim_{j} \sum_{k = j+1}^\infty \lambda_{\max}(M_k)^{1+\alpha}$, applying \cref{prop-max-eig} supplies the final result.
\end{proof}